\documentclass[pdflatex,sn-mathphys-num]{sn-jnl}

\usepackage{graphicx}%
\usepackage{multirow}%
\usepackage{amsmath,amssymb,amsfonts}%
\usepackage{amsthm}%
\usepackage{mathrsfs}%
\usepackage[title]{appendix}%
\usepackage{xcolor}%
\usepackage{textcomp}%
\usepackage{manyfoot}%
\usepackage{booktabs}%
\usepackage{algorithm}%
\usepackage{algorithmicx}%
\usepackage{algpseudocode}%
\usepackage{listings}%
\usepackage{float}
\usepackage{graphicx}
\usepackage{subcaption}
\usepackage{longtable}
\usepackage{multirow}
\usepackage{booktabs}

\theoremstyle{thmstyleone}%
\newtheorem{theorem}{Theorem}
\theoremstyle{thmstyletwo}%

\newtheorem{ass}{Assumption}[section]

\newtheorem{pfc}{Proof}[section]
\newtheorem{cor}{Corollary}[section]
\theoremstyle{thmstylethree}%

\begin{document}

\title[A Subsampling Based Neural Network for Spatial Data]{A Subsampling Based Neural Network for Spatial Data}


\author*[1]{\fnm{Debjoy} \sur{Thakur}}\email{debjoy@wustl.edu}



\affil*[1]{\orgdiv{Department of Statistics and Data Science}, \orgname{Washington Univeristy in St. Louis}, \orgaddress{ \city{St. Louis}, \postcode{63130}, \state{Missouri}, \country{USA}}}




\abstract{The application of deep neural networks in geospatial data has become a trending research problem in the present day. A significant amount of statistical research has already been introduced, such as generalized least square optimization by incorporating spatial variance-covariance matrix, considering basis functions in the input nodes of the neural networks, and so on. However, for lattice data, there is no available literature about the utilization of asymptotic analysis of neural networks in regression for spatial data. This article proposes a consistent localized two-layer deep neural network-based regression for spatial data. We have proved the consistency of this deep neural network for bounded and unbounded spatial domains under a fixed sampling design of mixed-increasing spatial regions. We have proved that its asymptotic convergence rate is faster than that of \cite{zhan2024neural}'s neural network and an improved generalization of \cite{shen2023asymptotic}'s neural network structure. We empirically observe the rate of convergence of discrepancy measures between the empirical probability distribution of observed and predicted data, which will become faster for a less smooth spatial surface. We have applied our asymptotic analysis of deep neural networks to the estimation of the monthly average temperature of major cities in the USA from its satellite image. This application is an effective showcase of non-linear spatial regression. We demonstrate our methodology with simulated lattice data in various scenarios.}




\maketitle

\section{Introduction}
In recent years, the application of machine learning and neural network models in spatial statistics has been a highly demanding and challenging research problem. Many researchers have focused on the application of neural networks in geospatial data sets for example pollution data, temperature data, spatial disease mapping, spatial econometrics, and so on. The neural network becomes popular day by day for its flexibility, and prediction accuracy. From a statistical point of view, many researchers focus on its explainability. For example \cite{schmidt2020nonparametric} has developed the convergence rate of a multilayer compositional neural network for independent and identically distributed (iid) feature space. \cite{kohler2021rate} has proved the optimum minimax rate of convergence of a fully connected deep neural network. \cite{kohler2023rate} has proved the asymptotic convergence of convolutional neural networks for image classifiers. \cite{braun2024convergence} has proved the asymptotic convergence rate of shallow neural networks. \cite{shen2023asymptotic} has proved the asymptotic convergence and asymptotic normality of the sieve neural network. These articles critically review the asymptotic convergence of neural networks under some regularity conditions.

Now coming to the application of neural networks in spatial data there are also some significant contributions. \cite{chen2020deepkriging} has introduced ``Deep Kriging'' (DK) by considering the basis functions in the input node of deep neural network similar concept of fixed rank kriging (FRK) of \cite{cressie2008fixed}. \cite{nag2023spatio} has extended this DK in a spatio-temporal framework. \cite{gerber2021fast} estimate parameter of Matern covariance function. \cite{zhan2024neural} has introduced a generalized least square-based graphical neural network for geospatial data. \cite{xiao2023nonparametric} has proposed a neural network-based spatial autoregressive model. \cite{wikle2023statistical} has reviewed almost all recently proposed different types of neural network-based models for spatial and spatio-temporal data. 

The literature has a significant gap regarding the asymptotic inference of neural network estimates for spatial data. From the literature, all the researchers mainly focus on solving two major problems, estimating mean function (\cite{chen2020deepkriging}), and parameter estimation of covariance function (\cite{gerber2021fast}). But \cite{zhan2024neural} has first proposed the spatial dependent graphical neural networking assuming the spatial points as nodes of neural network. But still, some open research questions arise.
\begin{itemize}
    \item[Question-1] \textit{How the neighborhood size $``M"$ contributes in asymptotic convergence of \cite{zhan2024neural}'s graphical neural network?}
    \item[Question-2] \textit{\cite{chen2020deepkriging}'s neural network considering \cite{nychka2015multiresolution}'s multi-resolution radial basis in the input nodes and \cite{nag2023spatio}'s long short time memory (LSTM) extension ignore the fact that for every spatial location the basis coefficients are different. As a result, it's not clear how they consider spatial random effects via neural networks.}
    \item[Question-3] \textit{\cite{xiao2023nonparametric} only considers the spatial random effect of response and covariate space are linearly additive ignoring the nonlinear strength of the neural network. Here how this network is convergent is unexplained.}
\end{itemize}
In this research, we have sorted these three problems of the previous research. We have proposed a nonlinear spatial neural network for spatial data which is a generalization of \cite{xiao2023nonparametric}'s spatial regression model. We have generalized \cite{shen2023asymptotic}'s single-layer sieve neural network into the deep neural network by increasing the width and depth. We explain how convergence is achieved by increasing neighborhood size under some regularity conditions with a faster convergence rate than that of \cite{zhan2024neural}. To overcome the shortcomings of considering spatial random effect in ``Deep Kriging" of \cite{chen2020deepkriging} we introduce a localized deep neural network. 

In Section~(\ref{sec:local}) we have described the spatial sampling design that is required for the inference of localized deep neural network and then we have discussed the details of the algorithm of localized deep neural network in spatial regression. In Section~(\ref{sec:asym}) we have discussed the asymptotic convergence, convergence rate, and central limit theorem of the weighted sum of localized deep neural network function in the mixed increasing domain under fixed sampling design. In Section~(\ref{sec:unc}) we explain the confidence interval of localized deep neural network functional, asymptotic behavior of KL divergence. In Section~(\ref{sec:result}) we empirically validate our theoretical results by different simulated data and one monthly average temperature data of major cities of the USA and satellite image.
\section{Methodology}
\subsection{Localized DNN}\label{sec:local}
Nowadays deep neural networks are famous for their approximation quality. For spatial framework, it's very common to assume the nodes in the location and then perform a neural network modeling. In this section, we will study a localized neural network that will apprehend the spatial dynamicity along with the relationship between response and covariates. This model is a potential generalization of spatial regression. This section is partitioned into two parts. The first part will explain how to import the spatial dependence and the second part will give an overview of how we can execute the algorithm. 
\subsubsection{Spatial Sampling Regions}
Let's consider an open connected set of $(-\frac{1}{2}, \frac{1}{2}]^{d}$ is $\mathcal{R}_{0}^{*}$ and a prototype Borel set $\mathcal{R}_0$ satisfying $\mathcal{R}_{0}^{*} \subset \mathcal{R}_0 \subset \text{Clo.}\mathcal{R}_0^{*}$. Let $\{\lambda_n\}_{n \in \mathbb{N}}$ be a sequence of non-decreasing positive real numbers ($\lambda_n = o(n)$) such that $\lambda_n \to \infty$ as $n \to \infty$. The sampling region $\mathcal{R}_n$ is achieved by `inflating' the prototype set $\mathcal{R}_0$ by a scaling factor $\lambda_n$
\[
\mathcal{R}_n = \lambda_n \mathcal{R}_0.
\]
As a result, the sampling region $\mathcal{R}_n$ becomes unbounded. Now for a mixed increasing spatial sampling design like \cite{lahiri2003central} we assume a diagonal matrix $\Delta = \text{diag}\left(e_1, e_2, \cdot \cdot \cdot, e_d\right)$ and $\mathcal{Z}^{d} = \{\Delta \mathbf{i}: \mathbf{i} \in \mathbb{Z}^{d}\}$ where $e_i \in (0, \infty)$ is an increment in the $i^{th}$ direction. Let $\{\eta_n\}_{n \in \mathbb{N}}$ be a sequence of decreasing positive real numbers ($\eta_n = o(\lambda_n^{-d})$) such that $\eta_n \to 0$ as $n \to \infty$. We consider the sampling sites $\{\mathbf{s}_1, \mathbf{s}_2, \cdot \cdot \cdot, \mathbf{s}_{N_{n}}\}$ under mixed-increasing domain case are given by the points
\begin{equation}\label{eq:sampling_point}
    \left\{\mathbf{s}_1, \mathbf{s}_2, \cdot \cdot \cdot, \mathbf{s}_{N_{n}}\right\} \equiv \{\mathbf{s} \in \eta_n \mathcal{Z}^{d}: \mathbf{s} \in \mathcal{R}_n\}.
\end{equation}
In (\ref{eq:sampling_point}) the sampling points of the scaled lattice $\eta_n \mathcal{Z}^{d}$ lie in the sampling region $\mathcal{R}_n$ and with increasing $n$ the lattice becomes finer, inflated (unbounded) and makes the region $\mathcal{R}_n$ dense. The maximum separation distance between two adjacent vertices of a lattice rectangle is bounded by $\eta_n\;\equiv \text{max}\{e_1, e_2,\cdot \cdot \cdot, e_d\} \to 0$ as $n \to \infty$. As a result, the sample size $n$ will satisfy the growth condition 
\begin{equation*}
    \lvert \Delta^{-1} \mathcal{R}_0 \rvert \left(\frac{\lambda_n}{\eta_n}\right)^{d} \sim N_n
\end{equation*}
of a larger order than the volume of $\mathcal{R}_n$. Next, we consider an increasing sequence of positive real-numbers $\{\delta_n\}_{n \in \mathbb{N}}$ such that $\delta_n = o(\log n)$ and $\delta_n \to \infty$ as $n \to \infty$. Let $\mathcal{N}_n(\mathbf{s}_i, \delta_n, \eta_n)$ is the neighborhood centered at $\mathbf{s}_i$ with radius $\delta_n$ like \cite{cressie2015statistics} which will satisfy
\[
\mathcal{N}_n(\mathbf{s}_i, \delta_n, \eta_n) \subset \mathcal{N}_{n+1}(\mathbf{s}_i, \delta_{n+1}, \eta_n).
\]
Assume $\{\Gamma_n\}_{n \in \mathbb{N}}$ is an increasing sequence of positive integers satisfying $\Gamma_n = \mathcal{O}(\lambda_n)$ that denotes the cardinality of spatial locations included in the neighborhood. As $n \to \infty$ the lattice spacing $\eta_n \to 0$ and, the cardinality, $\Gamma_n$, is going to $\infty$. Under mixed increasing sampling design the sampling region and neighborhood of every location have two properties: the neighborhoods of every location will become denser which incorporates the ``infill'' property of the spatial domain and the sampling region will be inflated with increasing $N_n$ which incorporates the ``increasing domain'' property of spatial domain. In the infill spatial neighborhood, we study the asymptotic behavior of a two-layer deep neural network (DNN) estimator. 

\subsubsection{Spatial DNN}
We assume that $\left\{Y(\mathbf{s}): \mathbf{s} \in \mathcal{R}_n\right\}$ and $\left\{\mathcal{X}(\mathbf{s}) \subset \mathbb{R}^{p}: \mathbf{s} \in \mathcal{R}_n \right\}$ are $(p+1)$ dimensional random fields (rf) on the bounded region $\mathcal{R}_n \subset \mathbb{R}^{d}$. Let $Y(\mathbf{s})$ and $\mathcal{X}(s)$  are respectively the spatial surface of response and predictor variables. In the spatial autoregression (SAR) model we consider 
\begin{equation}\label{eq:sar}
    Y(\mathbf{s}) = \rho W Y + \mathcal{X}(\mathbf{s}) \beta + \epsilon.
\end{equation}
In (\ref{eq:sar}) the dependent variable $Y \equiv Y(\mathbf{s}) \in \mathbb{R}^{N_n}$, the independent variables $\mathcal{X} \equiv \mathcal{X}(\mathbf{s}) \in \mathbb{R}^{N_n \times p}$. The spatial weight matrix $W \in \mathbb{R}^{N_n \times N_n}$ and the error $\epsilon(\cdot)$ is a white noise process with variance $\sigma^{2}$. In this article, we extend the SAR model in a non-linear framework. We assume there is a nonlinear regressive behavior in $Y$ for every spatial location such that 
\begin{equation}\label{eq:sieve_sar}
    \begin{split}
        &Y_t(\mathbf{s}_i) = f_0 \left( Y_t(\mathbf{s}_j), \mathcal{X}_t(\mathbf{s}_i), \mathcal{X}_t(\mathbf{s}_j) \right) + \epsilon_t(\mathbf{s}_i),\\
        &t=1,2, \cdot \cdot \cdot,n;\;\mathbf{s}_j \in \mathcal{N}_{\Gamma_n}(\mathbf{s}_i, \delta_n, \eta_n) \equiv \mathcal{N}_{n}(\mathbf{s}_i, \delta_n, \eta_n) \subset \mathcal{R}_n.
    \end{split}
\end{equation}
Here in (\ref{eq:sieve_sar}), we assume that the unknown function $f_0 \in \mathcal{F}$ such that,
\[
\mathcal{F} = \left\{f_2 \circ f_1: f_i \in \mathcal{M}_i^{\alpha_i, \Tilde{\beta}}\right\}
\]
where, $\mathcal{M}_i^{\alpha_i, \Tilde{\beta}}$ is the isotropic $\Tilde{\beta}$ Holder class of function $f_i$\footnotemark[1]. 
\footnotetext[1]{A function $f_i: \mathbb{R}^{d} \to \mathbb{R}$ will be from isotropic $\Tilde{\beta}$ Holder class $\mathcal{M}_i^{\alpha_i, \Tilde{\beta}}$ if there exists a constant $\mathcal{M}_i > 1;\; \alpha_i \in (0,1)$ such that 
\[
\lvert D^{\Tilde{\beta}} f_i(\mathbf{x}) - D^{\Tilde{\beta}} f_i(\mathbf{y}) \rvert \leq M_n\lvert\lvert \mathbf{x} - \mathbf{y} \rvert\rvert_2^{\alpha_i}
\]
where, $D^{\Tilde{\beta}} f(\mathbf{x})$ is the $\Tilde{\beta}$-th order partial derivative of $f_i(\mathbf{x})$.} In (\ref{eq:sieve_sar}) $Y_t(\mathbf{s}_i), \mathcal{X}_t(\mathbf{s}_i)$ are $n$ independent and identically distributed (iid) copies of $Y(\mathbf{s}_i)\in \mathbb{R}^{N_n}$ and $\mathcal{X}(\mathbf{s}_i) \in \mathbb{R}^{N_n \times p}$. Under infill asymptotic of fixed sampling design the nonlinear unknown function $f_{0}$ will minimize the population criterion function
\begin{equation}\label{eq:pop_crit}
    \begin{split}
        \mathcal{L}_{n}(f) &= \left[ \frac{1}{n} \sum_{t=1}^{n} \left\{Y_t(\mathbf{s}_i) - f(Y_t(\mathcal{N}_{\Gamma_n}(\mathbf{s}_i, \delta_n, \eta_n)), \mathcal{X}_t(\mathcal{N}_{\Gamma_n}(\mathbf{s}_i, \delta_n, \eta_n)))\right\}^{2}\right]\\
        &=  \frac{1}{n} \sum_{t=1}^{n}\left\{f(Y_t(\mathcal{N}_{\Gamma_n}(\mathbf{s}_i, \delta_n, \eta_n)), \mathcal{X}_t(\mathcal{N}_{n}(\mathbf{s}_i, \delta_n))) \right.\\
        &\quad\left. - f_0(Y_t(\mathcal{N}_{\Gamma_n}(\mathbf{s}_i, \delta_n, \eta_n)), \mathcal{X}_t(\mathcal{N}_{\Gamma_n}(\mathbf{s}_i, \delta_n, \eta_n))) \right\}^{2} + \sigma^{2}.
    \end{split}
\end{equation}
We can obtain a least square estimator of the regression function (\ref{eq:pop_crit}) by minimizing the empirical squared error loss $\mathbb{L}_{n}(f)$:
\begin{equation}\label{eq:emp_loss}
    \begin{split}
        \hat{f}_{n}(\mathbf{s}_i) &= \text{argmin}_{f \in \mathcal{F}} \mathbb{L}_{n}(f)\\
        &= \text{argmin}_{f \in \mathcal{F}} \frac{1}{n} \sum_{t=1}^{n} \left\{Y_t(\mathbf{s}_i) - f(Y_t(\mathcal{N}_{\Gamma_n}(\mathbf{s}_i, \delta_n, \eta_n)), \mathcal{X}_t(\mathcal{N}_{\Gamma_n}(\mathbf{s}_i, \delta_n, \eta_n)))\right\}^{2}.
    \end{split}
\end{equation}
The least-square estimator in (\ref{eq:emp_loss}) over universal $\mathcal{F}$ may face some undesired properties for example consistency. Therefore, optimizing the squared error loss over some less complex functional space $\mathcal{F}_{n}$ will be better. 
\begin{ass}\label{ass:rkhs}
    The functional space $\mathcal{F}_{n}$ with VC dimension $\nu$ and the functionals belonging to this space are square-integrable (one choice might be Reproducing Kernel Hilbert Space).  
\end{ass}
Here the function space $\mathcal{F}_{n}$ in Assumption.~\ref{ass:rkhs} is an approximation of $\mathcal{F}$ and this approximation error tends to $0$ with increasing sample size in $\mathcal{N}_{\Gamma_n}(\mathbf{s}_i, \delta_n, \eta_n)$. We consider a sequence of non-decreasing classes of functions
\[
\mathcal{F}_{1} \subset \mathcal{F}_{2} \subset \cdot \cdot \cdot \subset \mathcal{F}_{n}\subset \mathcal{F}_{n+1} \subset \cdot \cdot \cdot \subset \mathcal{F}
\]
where $\bigcup_{n\geq 1} \mathcal{F}_{n}$ is dense in $\mathcal{F}$. Moreover for any $f\in \mathcal{F}$ there exists a projection $\pi_{n}f \in \mathcal{F}_n$ such that some pseudometric defined in $\mathcal{F}$ i.e. $d(f,\pi_{n}f) \to 0$ where as $n \to \infty$. This type of function is defined as a 2-DNN according to \cite{grenander1981abstract}. The DNN estimator $\hat{f}_{n}$ would satisfy
\[
\mathbb{L}_{n}(\hat{f}_{n}) \leq \text{Inf}_{f \in \mathcal{F}_n} \mathbb{L}_{n}(f) + \mathcal{O}_{P}(\Xi_n)
\]
with $lim_{n \to \infty} \Xi_n = 0$. We will describe the localized DNN with two hidden layers and a \texttt{tanh} activation function belonging to the isotropic holder class\footnotemark[1]. We assume a two-layer DNN with $r$ hidden nodes:
\begin{equation}\label{eq:2_SNN}
    \begin{split}
        &\mathcal{F}_{r,n} \equiv \left\{f_2 \circ f_1: f_{1} \in \mathcal{F}_{r}^{(1)} \subset \mathcal{M}_1^{\alpha_1, \Tilde{\beta}}\; \& \;f_{2} \in \mathcal{F}_{r}^{(2)} \subset \mathcal{M}_2^{\alpha_2, \Tilde{\beta}}\right\}\\
        &\equiv \left\{ \vartheta_0 + \sum_{i=1}^{r} \vartheta_i \texttt{tanh}\Bigg(\sum_{j=1}^{r} \nu_{ij}\texttt{tanh}\Bigg(\sum_{k=1}^{p}\sum_{l=1}^{\Gamma_n} \upsilon_{ijkl} X_k(\mathbf{s}_{\sigma_l}) + \right.\\
        &\left. \sum_{l=1}^{\Gamma_n}\upsilon_{ij(p+1)l} Y(\mathbf{s}_{\sigma_l}) + \upsilon_{ij00}\Bigg) + \nu_{i0}\Bigg):
        \right.\\
        &\left.
       \sum_{i=0}^{r} \lvert \vartheta_i\rvert \leq V_2;\sum_{j=0}^{r} \lvert \nu_{ij}\rvert \leq V_2 \;\text{max}_{\substack{1\leq i \leq r \\ 1\leq j \leq r}} \sum_{k=0}^{p+1}\sum_{l=1}^{\Gamma_n} \lvert \upsilon_{ijkl} \rvert \leq V_1 \text{for some}\; V_1 > 1 \right\};\\
       &\text{for some}\; \text{min}\{V_1, V_2\} > 1; \vartheta_i, \nu_{ij}, \upsilon_{ijkl} \in \mathbb{R};\\
       &\;i,j= 0,\cdot \cdot \cdot, r;k= 0,\cdot \cdot \cdot,(p+1);\;l=1,2,\cdot \cdot \cdot, \Gamma_n;\; \sigma_l\in \{1,2,\cdot \cdot \cdot, n\}.
    \end{split}
\end{equation}
Assume $\lvert \lvert Y_i\rvert \rvert_{\infty} \leq \text{min}\{V_1, V_2\} \Rightarrow \lvert \lvert f_{0}\rvert \rvert_{\infty} \leq \text{min}\{V_1, V_2\}$. In (\ref{eq:2_SNN}) we assume $\left\{\mathbf{s}_{\sigma_1},\mathbf{s}_{\sigma_2}, \cdot \cdot \cdot, \mathbf{s}_{\sigma_{\Gamma_n}}\right\} \in \mathcal{N}_{\Gamma_n}(\mathbf{s}_i, \delta_n, \eta_n)$.
\begin{figure}
    \centering
    \includegraphics[width=\linewidth]{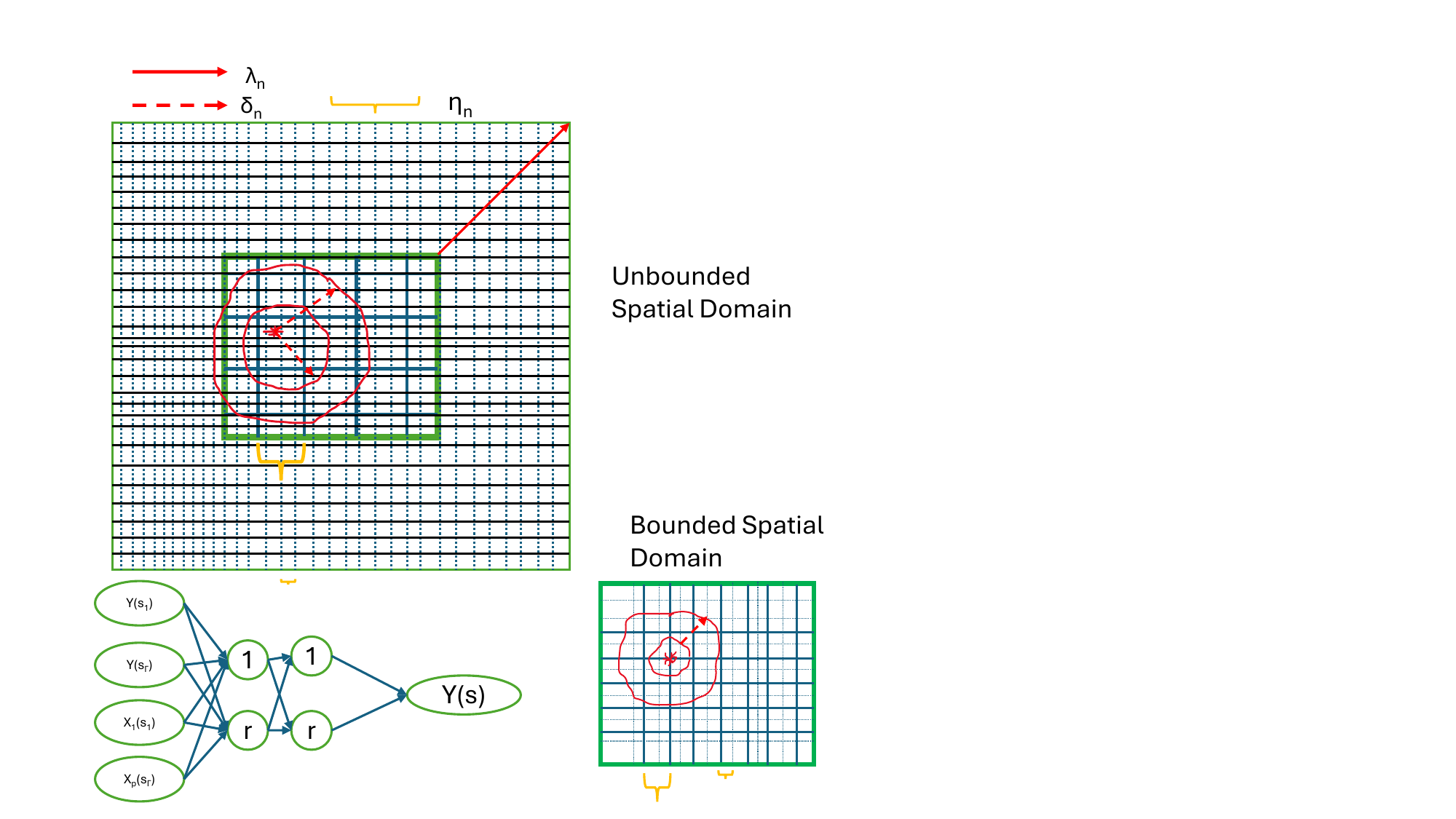}
    \caption{Flowchart of Localized DNN Estimator}
    \label{fig:neighbor}
\end{figure}
In Figure~\ref{fig:neighbor} we have illustrated the workflow of a two-layer DNN-based spatial regression model. For unbounded regions corresponding to every target location we create a spatial neighborhood of radius $\delta_n$ and consider those locations included in this spatial neighborhood. We feed the response and covariates of those spatial locations in the input node of a DNN having two hidden layers. Here to avoid high-dimensional problems we also give one constraint that $(p+1)\Gamma_n = o(\text{max}\{N_n, n\})$. This restriction is crucial for bounded and unbounded spatial domains. However, for the bounded domain case, it will be equivalent to the purely ``infill'' spatial subsampling. In Algorithm~(\ref{alg:dnn}) we have explained the algorithm 2-DNN based spatial regression model. From this Algorithm~(\ref{alg:dnn}) it's clearly understandable that this model is capturing spatial dynamicity at the local level by finer grid spacing concept of multi-resolution and overall dynamicity by the sieve class of nonlinear functionals. 
\begin{algorithm}
\caption{2-DNN Spatial Auto Regression}
\begin{algorithmic}[1]
\small
\State Spatial sampling region, $R_n = \lambda_n \cdot R_0$.
\State Initial Grid Spacing in Lattice, $\eta_n \gets \frac{1}{\lambda_n + d}$.
\For{l in 1:L}
    \State Making finer grid $\eta_{nl} \gets \frac{\eta_n}{2^{l}}$.
    \For{ each $\delta_n$}
    \State Create Neighborhood $\mathcal{N}_{\Gamma_n}(\mathbf{s}_i, \delta_n, \eta_n)$ and store $Y, \mathcal{X}$ in $\mathcal{Y}_t(\mathbf{s}_i,\mathcal{N}_{\Gamma_n}).$
    \For{i in $1: N_n$}
    \State Define 2-DNN spatial regression model with \texttt{tanh} activation:
    \begin{equation*}
        \begin{split}
           Y_i &= f_2(W_2 f_1(W_1 \mathcal{Y}_t(\mathbf{s}_i,\mathcal{N}_{\Gamma_n}) + b_1) + b_2),\\
           h_i &= f(W_i h_{i-1} + b_i),\; 
      \tanh(x) = \frac{e^x - e^{-x}}{e^x + e^{-x}}.   
        \end{split}
    \end{equation*}
\For{epoch = 1 to num\_epochs}
    \State Shuffle the training data.
    \For{b in 1:B}
        \State Compute loss for minibatch.
       \[\mathbb{L}_{nb}= \frac{1}{n} \sum_{t=1}^{n} \left\{Y_t(\mathbf{s}_i) - f(\mathcal{Y}_t(\mathbf{s}_i,\mathcal{N}_{\Gamma_n}))\right\}^{2}.
       \]
    \State Compute gradients for minibatch using backpropagation.
       
\[
\frac{\partial\; \mathbb{L}_{nb}}{\partial w} = \frac{1}{n} \sum_{i=1}^n (Y_i - \hat{Y}_i) \cdot \frac{\partial \hat{Y}_i}{\partial w}
\]
 
    \State Update model parameters using Adam optimizer.
        \begin{align*}
m^{*}_p &= \beta^{*}_1 m^{*}_{p-1} + (1 - \beta^{*}_1) g^{*}_p; 
v^{*}_p = \beta^{*}_2 v^{*}_{p-1} + (1 - \beta^{*}_2) g^{*2}_p \\
\hat{m}^{*}_p &= \frac{m^{*}_p}{1 - \beta^{*p}_1}; 
\hat{v}^{*}_p = \frac{v_p^{*}}{1 - \beta^{*p}_2} ;
\theta^{*}_{p+1} = \theta^{*}_p - \frac{\alpha^{*} \hat{m}^{*}_p}{\sqrt{\hat{v}^{*}_p} + \epsilon^{*}}
\end{align*}
 \State Update learning rate according to necessity.
    \EndFor
    \State Evaluate the model's performance on the validation data set.
    \EndFor
\EndFor
    \EndFor
\EndFor

\end{algorithmic}
\label{alg:dnn}
\end{algorithm}

\subsection{Asymptotic Analysis}\label{sec:asym}
This section is devoted to discussing the asymptotic behavior of a two-layer DNN (2-DNN) estimator. This section is divided into two parts. The first part will study the existence, consistency, and convergence rate of a 2-DNN. 
\subsubsection{Asymptotic Convergence}
 The first theorem will prove the existence of this 2-DNN estimator. This theorem is motivated by Theorem 2.2 of \cite{white1991some} and an extension of Corollary 2.1 of \cite{shen2023asymptotic}. Here the consistency of 2-DNN is different from \cite{shen2023asymptotic}'s consistency of single-layer sieve neural network (SNN) because \cite{shen2023asymptotic} has achieved the consistency for single-layer DNN with \texttt{sigmoid} activation function. A single-layer DNN is not very effective in dealing with complex spatial problems. But here we propose the result for 2-DNN  which is deep and helpful for real data. Here the rate of increase of input nodes and nodes in the hidden layer is slower than the rate of inflation, $\lambda_n$ of the spatial sampling region. 
\begin{theorem}\label{thm:existance}[Existence]
    Let $(\Omega, \mathcal{A}, P)$ be a complete probability space and let $(\mathcal{F}, \rho)$ be a pseudo-metric. Let $\{\mathcal{F}_{r,n}\}$ be a sequence of compact subsets of $\mathcal{F}$. Let $\mathbb{L}_{n}:\Omega \times \mathcal{F}_{r,n} \to \Bar{\mathbb{R}}$ be $\mathcal{A} \otimes \mathcal{B}(\mathcal{F}_{r,n})/ \mathcal{B}(\Bar{\mathbb{R}})$ measurable and suppose for every $\omega \in \Omega$ if $\mathbb{L}_{n}(\omega,\cdot)$ is lower semicontinuous on $\Theta_{n}$. Then for every $n = 1,2, \cdot \cdot \cdot$ there exists a 2-DNN estimator $\hat{f}_{n}: \Omega \to \mathcal{F}_{r,n}$ which is $\mathcal{A}/\mathcal{B}(\mathcal{F}_{r,n})$ measurable such that $\mathbb{L}_{n}(\hat{f}_{n}(\omega)) = \text{Inf}_{f \in \mathcal{F}_{r,n}} \mathbb{L}(f)$.
\end{theorem}
\begin{proof}
    Pleas see the proof in Appendix~(\ref{sec:pf_ex})
\end{proof}
Theorem~\ref{thm:existance} identifies the condition of existence of 2-DNN. The following theorem will justify the convergence of this two-layer DNN.
\begin{theorem}\label{thm:convergence}[Convergence]
\begin{itemize}
    \item[A2.1] If $\Gamma_n = o\left(n^{\psi}\right), \lambda_n = o(n), r = o\left(n^{\beta}\right),V_2\sqrt{\sigma^{2} + 1} < < r^{2}, V_1 V_2 = o(\sqrt{\log n})$ for some $0< \psi + \beta < 1$.
    \item[A2.2] The functional class $\mathcal{F}_{r,n}$ with Vapnik-Chervonenkis (VC) dimension $\nu$ with bounded envelope function $\Tilde{g}$. This bounded envelop function, $\Tilde{g}$ is uniformly bounded by min$\{V_1, V_2\}$.
    \item[A2.3] The Covering number that grows at a controlled rate such that
    \[
    \log \mathcal{N}\left(\varphi, \mathcal{F}_{r,n}, \lvert \lvert \cdot \rvert \rvert \right) \leq \nu C \varphi^{-d},
    \]
    for some $C > 0$ then, 
\end{itemize}
    \begin{equation*}
        \lvert \lvert \hat{f_{n}} - f_{0} \rvert \rvert_{n} \to^{P} 0\;\; \text{as}\;\; n\to \infty
    \end{equation*}
\end{theorem}
\begin{proof}
    Please see the proof in Appendix~(\ref{sec:pf_conv})
\end{proof}
Theorem~\ref{thm:convergence} is the first theorem that has established the convergence of two-layer DNN after \cite{schmidt2020nonparametric}. But there are two key differences one is we extend DNN in spatial design and the second is our covering number's growth rate is restricted. The next theorem will comment on the convergence rate of this DNN model.
\begin{theorem}\label{thm:conv_rate}
If (A2.1)-(A2.3) from Theorem~\ref{thm:convergence} hold and,
    \begin{equation}
        \begin{split}
             \zeta_n&=\mathcal{O}_{P}\left(min\left\{ \left\lvert\lvert\pi_{r,n} f_{0} - f_{0}\right\rvert\rvert_{n}^{2},\right.\right. \\
            &\left. \left. \frac{\left(r^{2} + r ((p+1)\Gamma_n+4) +2\right)\log\left[\left(r^{2} + r ((p+1)\Gamma_n+4) +2\right)(V_1V_2)^{2}\right]}{n},\right.\right.\\
         &\left.\left. \frac{\left(r^{2} + r ((p+1)\Gamma_n+4) +2\right)\log (n\log n)}{n} \right\}\right)\\  
        \end{split}
    \end{equation}
    
    then, 
    \begin{equation*}
        \begin{aligned}
           \lvert \lvert \hat{f}_n - f_0 \rvert \rvert_{n}
            =\mathcal{O}_{P}&\left\{\text{max}\;\left[\left\lvert\lvert\pi_{r,n} f_{0} - f_{0}\right\rvert\rvert_{n}, 
         \left(\frac{n}{n^{\psi + \beta} \log (n^{\psi + \beta} \log n)}\right)^{-1/2},\right.\right.\\
         &\left.\left.
         \left(\frac{n}{n^{\psi + \beta} \log (n\log n )}\right)^{-1/2}\right]\right\}.            
        \end{aligned}
    \end{equation*}
\end{theorem}
\begin{proof}
    Please see the proof in Appendix~(\ref{sec:pf_convrt})
\end{proof}
In Theorem~\ref{thm:conv_rate} we have proved the convergence rate of 2-DNN for every spatial location with its increasing neighborhood ($\mathcal{N}_{\Gamma_n}(\mathbf{s}_i, \delta_n, \eta_n)$) size, $\Gamma_n$. This convergence rate is faster than the first convergence rate of spatial DNN of \cite{zhan2024neural}. 

\subsection{Uncertainity Quantification}\label{sec:unc}
This section discusses the subsampling-based confidence interval of the 2-DNN functional of a specified location. We have already observed that for every location 2-DNN becomes a random vector. Therefore here our main goal is to find $\mathcal{C}_{\alpha} = [L_\alpha, U_\alpha]$. Let's consider $\mathcal{Y}_{\Gamma_j} = y(\mathcal{N}_{\Gamma_j}(\mathbf{s}_i, \delta_n, \eta_n))$ for zero means and for nontrivial mean function $\mathcal{Y}_{\Gamma_j} = \left(y(\mathcal{N}_{\Gamma_j}(\mathbf{s}_i, \delta_j)), \mathcal{X}(\mathcal{N}_{\Gamma_j}(\mathbf{s}_i, \delta_j))\right)$ such that
\[
P( f_0(\mathcal{Y}_{\Gamma_j}) \in \mathcal{C}_{\alpha} ) = 1-\alpha.
\]
To find this $\mathcal{C}_\alpha$ we assume that 
\[
\text{Var}\left(n^{\lambda} \hat{f}_n(\mathcal{Y}_{\Gamma_j}) \right) \to 0;\;\text{as}\; n \to \infty,\; \text{for some}\; \lambda > 0.
\]
We take a subsample of spatial sampling regions at each iteration of size $\Gamma_j$ with corresponding radius $\delta_j$ where $\Gamma_j < \Gamma_{j+1} << \lvert \mathcal{R}_n \rvert \equiv N_n$. Its corresponding neighborhood size, $\delta_j > \delta_{j-1} > \eta_n\;\forall\; j=1,2, \cdot \cdot \cdot, B$. From each sampling block, we fed the sampling points included in this block and got an estimate of the 2-DNN functional corresponding to the location for each $\Gamma_j$ i.e. $\hat{f}_{n,j}$ (please see Figure~\ref{fig:neighbor}). In this manner take $B$ sub-samples and got the estimate $\hat{f}_{n1}(\mathbf{s}_i),\hat{f}_{n2}(\mathbf{s}_i),\cdot \cdot \cdot, \hat{f}_{nB}(\mathbf{s}_i)$. Then the corresponding confidence interval will be 
\begin{equation}\label{eq:ci}
    \mathcal{C}_{\alpha}(\mathbf{s}_i) \equiv \Bar{\hat{f}}_n(\mathbf{s}_i) \pm z_{1-\alpha/2} \Lambda(\mathbf{s}_i).
\end{equation}
In (\ref{eq:ci}) the first term of RHS is $\Bar{\hat{f}}_n(\mathbf{s}_i) \equiv \frac{1}{B} \sum_{j=1}^{B} \hat{f}_{nj}(\mathbf{s}_i)$ and 
\[
\Lambda(\mathbf{s}_i) \equiv \sqrt{\frac{1}{B} \sum_{j=1}^{B} \left(y(\mathbf{s}_i) - \hat{f}_{nj}(\mathbf{s}_i)\right)^{2}}.
\]
This confidence interval in (\ref{eq:ci}) briefs the idea of how the 2-DNN varies w.r.t the sample size included in the neighborhood and the inherited spatial dynamicity in the DNN functional estimate. Next, we will discuss the asymptotic behavior of KL divergence of the empirical distribution of observed and predicted responses with increasing subsample sizes. In (\ref{eq:sieve_sar}) we have already described that 
\[
y_{t}(\mathbf{s}_i) = f_0\left(y_t(\mathbf{s}_{i_1}), \cdot \cdot \cdot, y_t(\mathbf{s}_{i_m}) \right) + \epsilon_t, \; t = 1,2,\cdot \cdot \cdot, n.
\]
Here $\left\{y_1(\mathbf{s}_i), y_2(\mathbf{s}_i), \cdot \cdot \cdot, y_n(\mathbf{s}_i)\right\}$ are iid copies of $y(\mathbf{s}_i)$. According to (\ref{eq:emp_loss}) we know that
\[
\hat{f}_{n,m} = \text{argmin}_{f\in\mathcal{F}} \frac{1}{n} \sum_{t=1}^{n} \left(y_t(\mathbf{s}_i) - f\left(y_t(\mathbf{s}_{i_1}), \cdot \cdot \cdot, y_t(\mathbf{s}_{i_m}) \right)\right)^{2}
\]
and $m(n) \equiv m \equiv \Gamma_n \equiv \mathcal{O}\left(n^{\psi}\right)$ for some $0< \psi < 1$ mentioned in (A2.1) of Theorem~\ref{thm:convergence}. As a consequence of Theorem~\ref{thm:convergence}, it's clear that $\hat{f}_{nm}$ is consistent. In the next Corollary, we will discuss the asymptotic convergence of KL divergence of the empirical distribution of observed and predicted spatial surfaces. Assume the empirical distribution of $y(\mathbf{s}_i)$ and $\hat{y}(\mathbf{s}_i)$ are respectively
\begin{equation*}
    \begin{split}
        \Tilde{p}_{n, y(\mathbf{s}_i)}(y) &= \frac{1}{n} \sum_{t=1}^{n} I\left(y_t(\mathbf{s}_i) \leq y \right),\\
        \hat{\Tilde{p}}_{n,m, \hat{y}(\mathbf{s}_i)}(y) &= \frac{1}{n} \sum_{t=1}^{n} I\left(\hat{y}_{t,m}(\mathbf{s}_i) \leq y \right).\\
    \end{split}
\end{equation*}
\begin{cor}
If $m(n) = n^{\beta}$ for some $\beta \in (0,1)$ and Theorem~\ref{thm:convergence} is valid then 
\[
 KL\left( \Tilde{p}_{n, y(\mathbf{s}_i)}(y) \left\lvert\right\rvert \hat{\Tilde{p}}_{n,m, \hat{y}(\mathbf{s}_i)}(y)  \right) \to 0,\;\text{whenever}\; m(n) \to \infty.
\]
\end{cor}
\begin{pfc}
    Under the consideration of Theorem~\ref{thm:convergence} and some regularity conditions mentioned in Theorem~\ref{thm:convergence} 
    \[
     \Tilde{p}_{n, y(\mathbf{s}_i)}(y) \to p_{\infty}\left(y; \theta_{\infty}\right),\; \text{as}\; n \to \infty.
    \]
    WLG assume $\exists$ a random set $\mathcal{Q}_m = \left\{m:\; \lvert \hat{y}_{t,m}(\mathbf{s}_i) - y(\mathbf{s}_i) \rvert \geq \varepsilon\; \text{for some}\; \varepsilon > 0\right\}$ with measure $P$ such that 
    \[
    P(\mathcal{Q}_m) \to 0\; \text{as}\; m \to \infty.
    \] 
    Then from Theorem~\ref{thm:conv_rate} we can say 
    \[
    \hat{\Tilde{p}}_{n,m, \hat{y}(\mathbf{s}_i)}(y \vert \mathcal{Q}_m) \to p_{\infty}\left(y; \theta_{\infty}\right),\; \text{as}\; n \to \infty.
    \]
 It suffices the proof.   
\end{pfc}
In this article, we empirically validate the KL divergence will asymptotically converge to zero with increasing $\delta_n$.

\section{Results and Discussion}\label{sec:result}
\subsection{Consistency for Unbounded Spatial Domain: Simulation Study}
We have simulated lattice data from the Matern covariance function \cite{cressie2015statistics} to observe the asymptotic convergence of 2-DNN in the mixed increasing domain setup. The well-known Matern covariance function is 
\[
\mathcal{C}_{\nu}(h) = \sigma^{2} \frac{2^{1-\nu}}{\Gamma(\nu)} \left(\sqrt{2\nu}\frac{h}{\phi}\right)^{\nu}\mathcal{K}_{\nu}\left(\sqrt{2\nu}\frac{h}{\phi}\right).
\]
In this Matern covariance function $\sigma^{2}, \phi, \kappa$ denote sill, range, and smoothing parameters respectively. The parameter space is $\{(\nu, \phi, \sigma^{2}): \nu > 0, \sigma^{2} > 0, \phi > 0\}$. We are considering $\eta_n = \frac{1}{\lambda_n + d}$ and assume $\Gamma_n = n^{\psi}$ with $\psi \in (0,1)$. We study our research outcomes for different radius $\delta_n = 0.3, 0.6, 0.8$ of neighborhood $\mathcal{N}_{\Gamma_n}(\mathbf{s}_i; \delta_n, \eta_n)$. Assume $N_n$ is the number of lattice data points and the 10 covariate values such as $X_1, X_2, \cdot \cdot \cdot, X_{10}$, and $Y = \sum_{i=1}^{10} i X_i$. This neural network will train $70\%$ graphs as train, $20\%$ as validation, and $10\%$ as test set. In Table~(\ref{tab:consist}) we have empirically validated the asymptotic convergence of neural network for different parameter combinations of smoothing parameters with increasing neighborhood size in fixed sampling design under mixed increasing domain framework in four scenarios. We have observed the mean square prediction error (MSPE) relatively becomes smaller with increasing $\delta_n$ and decreasing $\eta_n$ from Table~(\ref{tab:consist}). In Figure~\ref{fig:ci_snn} we demonstrate the $95\%$ CI of localized 2-DNN functional for a particular spatial location in the lattice. From Figure~\ref{fig:sub05} it's visible that the CI width becomes narrower with increasing sample size and a similar pattern is detectable for Figure~\ref{fig:sub1}, Figure~\ref{fig:sub15}, and Figure~\ref{fig:sub2}. Likewise, empirical CDF of spatial 2-DNN functional of predicted data converges rapidly with increasing $\delta_n$ since $KL\left( \Tilde{p}_{n, y(\mathbf{s}_i)}(y) \left\lvert\right\rvert \hat{\Tilde{p}}_{n,m, \hat{y}(\mathbf{s}_i)}(y)  \right) \to 0$ with increasing $m, n$ from Figure~\ref{fig:kl_snn}. In Figure~\ref{fig:kl15}, Figure~\ref{fig:kl2} we detect that the convergence rate is faster than the Figure~\ref{fig:kl05} and Figure~\ref{fig:kl1} with increasing smoothness parameter $\kappa$. From this entire discussion, we can get a brief scenario of the convergence of localized 2-DNN. From Figure~\ref{fig:kl_kappa}, we detect that under the assumption of increasing neighborhood size of $o(n^{\psi})$ the localized 2-DNN functional converges rapidly with decreasing smoothing parameter $\kappa$ of Matern variogram.
\begin{table*}[htbp]
    \centering
    \caption{Simulated Lattice Data from Matern Variogram}
    \label{tab:consist}
    \scalebox{0.85}{ 
    \begin{tabular}{|c|c|c|}
        \toprule
        \multirow{2}{*}{\textbf{Scenario 1}} & \multicolumn{2}{c|}{$\sigma^{2} = 1, \phi = 0.1, \kappa = 0.5$} \\
        
        & $\delta_n$ & MSPE \\
        \midrule
        \multirow{3}{*}{$\lambda_n = 4, n = 20, \eta_n = 0.16$} & 0.3 & 0.046 \\
        & 0.6 & 0.0035 \\
        & 0.8 & 0.00137 \\
        \hline
        \multirow{3}{*}{$\lambda_n = 5, n = 35, \eta_n = 0.14$} & 0.3 & 0.042 \\
        & 0.6 & 0.0039 \\
        & 0.8 & 0.00138 \\
        \hline
        \multirow{3}{*}{$\lambda_n = 6, n = 50, \eta_n = 0.12$} & 0.3 & 0.04 \\
        & 0.6 & 0.0065 \\
        & 0.8 & 0.00131 \\
        \midrule
        \multirow{2}{*}{\textbf{Scenario 2}} & \multicolumn{2}{c|}{$\sigma^{2} = 1, \phi = 0.1, \kappa = 1$} \\
       
        & $\delta_n$ & MSPE \\
        \midrule
        \multirow{3}{*}{$\lambda_n = 4, n = 20, \eta_n = 0.16$} & 0.3 & 0.05 \\
        & 0.6 & 0.0047 \\
        & 0.8 & 0.0007 \\
        \hline
        \multirow{3}{*}{$\lambda_n = 5, n = 35, \eta_n = 0.14$} & 0.3 & 0.047 \\
        & 0.6 & 0.0045 \\
        & 0.8 & 0.00097 \\
        \hline
        \multirow{3}{*}{$\lambda_n = 6, n = 50, \eta_n = 0.12$} & 0.3 & 0.04 \\
        & 0.6 & 0.0063 \\
        & 0.8 & 0.0012 \\
        \bottomrule
         \multirow{2}{*}{\textbf{Scenario:3}} & \multicolumn{2}{c|}{$\sigma^{2} = 1, \phi = 0.1, \kappa = 1.5$} \\
        & $\delta_n$ & MSPE \\
        \hline
        \multirow{3}{*}{$\lambda_n = 4, n = 20, \eta_n = 0.16$} & 0.3 & 0.053 \\
        & 0.6 & 0.0059 \\
        & 0.8 & 0.00153 \\
        \hline
        \multirow{3}{*}{$\lambda_n = 5, n = 35, \eta_n = 0.14$} & 0.3 & 0.048 \\
        & 0.6 & 0.006 \\
        & 0.8 & 0.0024 \\
        \hline
        \multirow{3}{*}{$\lambda_n = 6, n = 50, \eta_n = 0.12$} & 0.3 & 0.036 \\
        & 0.6 & 0.0063 \\
        & 0.8 & 0.00125 \\
        \hline
        \multirow{2}{*}{\textbf{Scenario:4}} & \multicolumn{2}{c|}{$\sigma^{2} = 1, \phi = 0.1, \kappa = 2$} \\
        & $\delta_n$ & MSPE \\
        \hline
        \multirow{3}{*}{$\lambda_n = 4, n = 20, \eta_n = 0.16$} & 0.3 & 0.047 \\
        & 0.6 & 0.00629 \\
        & 0.8 & 0.00113 \\
        \hline
        \multirow{3}{*}{$\lambda_n = 5, n = 35, \eta_n = 0.14$} & 0.3 & 0.047 \\
        & 0.6 & 0.0051 \\
        & 0.8 & 0.0022 \\
        \hline
        \multirow{3}{*}{$\lambda_n = 6, n = 50, \eta_n = 0.12$} & 0.3 & 0.04 \\
        & 0.6 & 0.0064 \\
        & 0.8 & 0.00134 \\
        \hline    
    \end{tabular}%
    }
\end{table*}
\begin{figure}[H]
    \centering
    \begin{subfigure}[b]{0.8\linewidth}
        \centering
        \includegraphics[width=\linewidth]{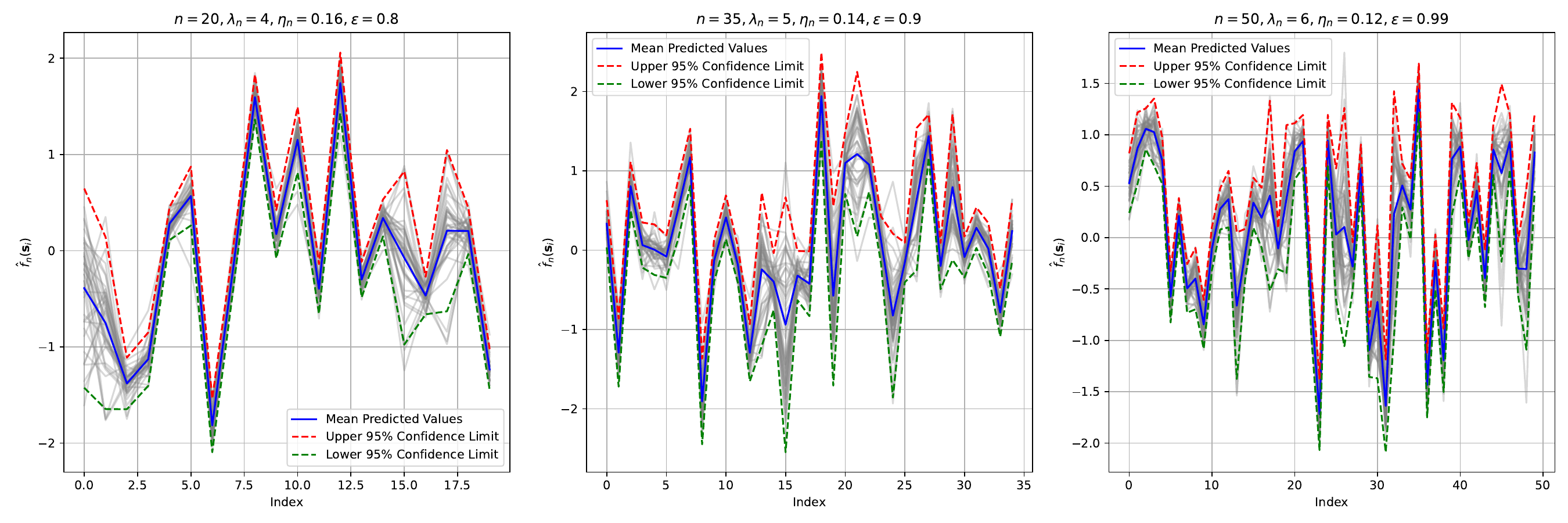}
        \caption{$(\sigma^{2}, \phi, \kappa)$ = $(1,0.1, 0.5)$ for n = $20\;(left),35\;(middle),50\;(right)$.}
        \label{fig:sub05}
    \end{subfigure}
    
    
    \begin{subfigure}[b]{0.8\linewidth}
        \centering
        \includegraphics[width=\linewidth]{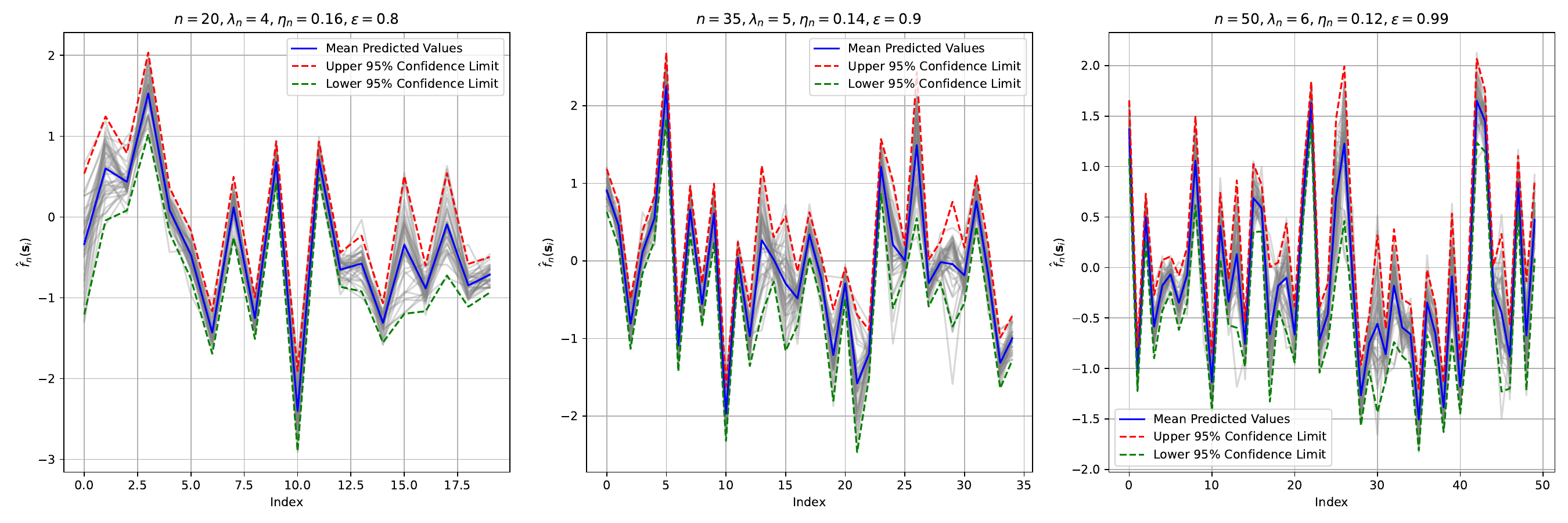}
        \caption{$(\sigma^{2}, \phi, \kappa)$ = $(1,0.1, 1)$ for n = $20\;(left),35\;(middle),50\;(right)$.}
        \label{fig:sub1}
    \end{subfigure}
    
    
    \begin{subfigure}[b]{0.8\linewidth}
        \centering
        \includegraphics[width=\linewidth]{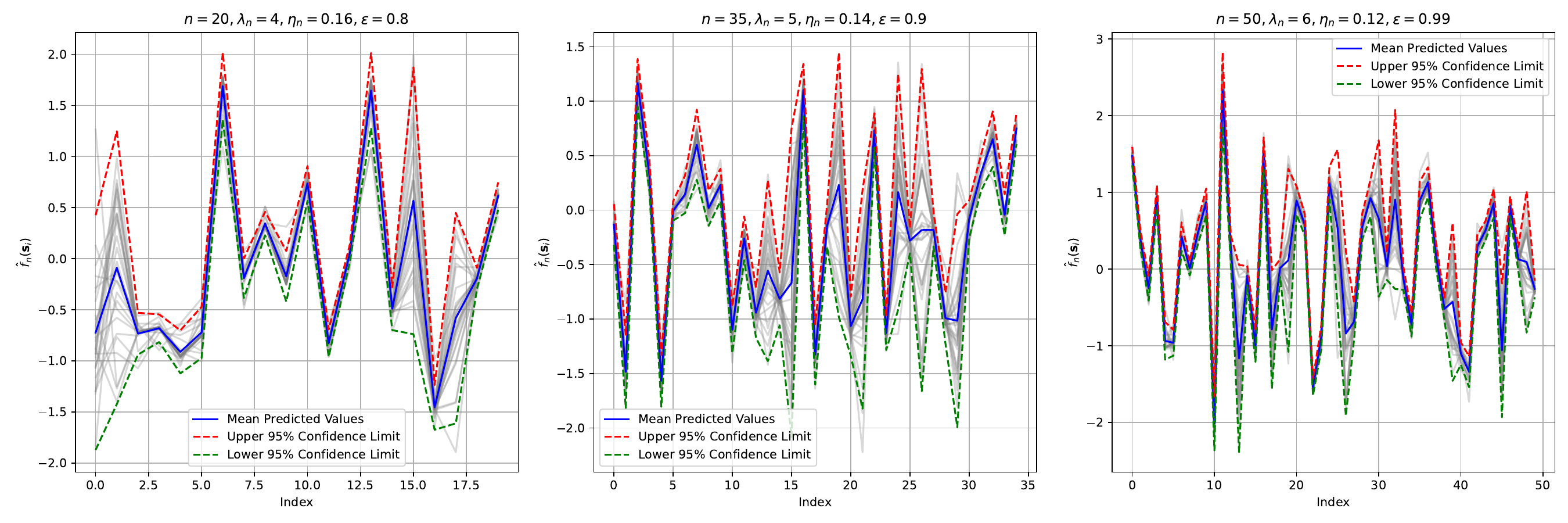}
       \caption{$(\sigma^{2}, \phi, \kappa)$ = $(1,0.1, 1.5)$ for n = $20\;(left),35\;(middle),50\;(right)$.}
        \label{fig:sub15}
    \end{subfigure}
    
    
    \begin{subfigure}[b]{0.8\linewidth}
        \centering
        \includegraphics[width=\linewidth]{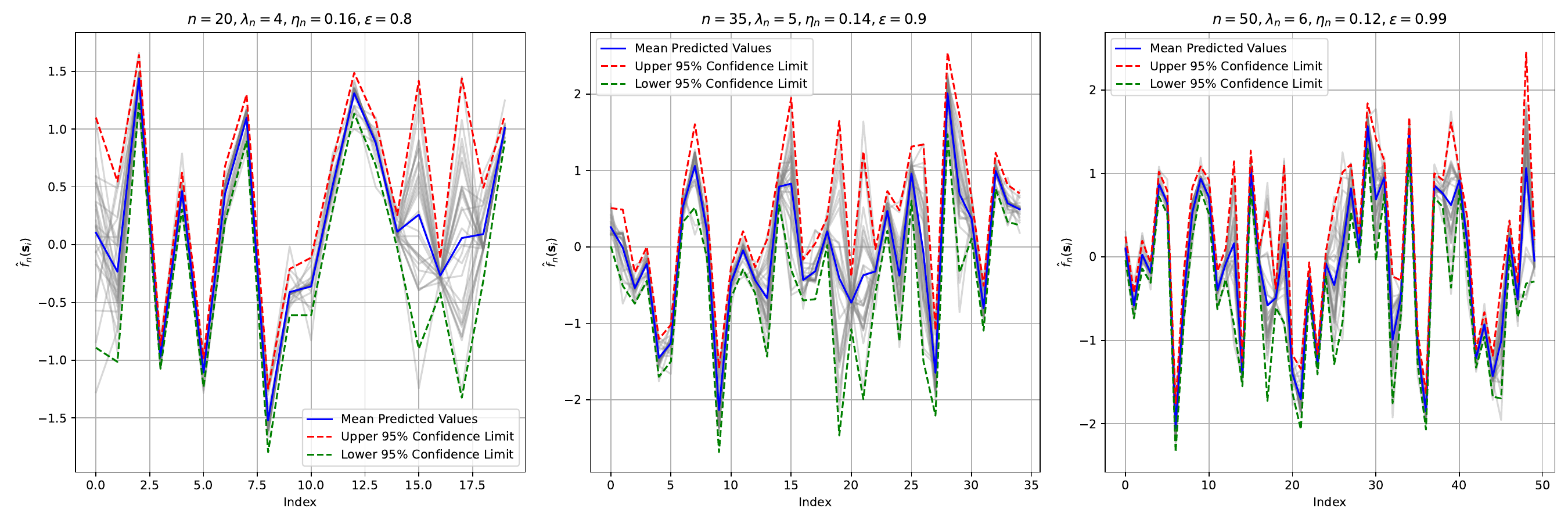}
        \caption{$(\sigma^{2}, \phi, \kappa)$ = $(1,0.1, 2)$ for n = $20\;(left),35\;(middle),50\;(right)$.}
        \label{fig:sub2}
    \end{subfigure}
    
    \caption{95$\%$ CI of localized 2-DNN functional.}
    \label{fig:ci_snn}
\end{figure}
\begin{figure}[H]
    \centering
    \begin{subfigure}[b]{0.8\linewidth}
        \centering
        \includegraphics[width=\linewidth]{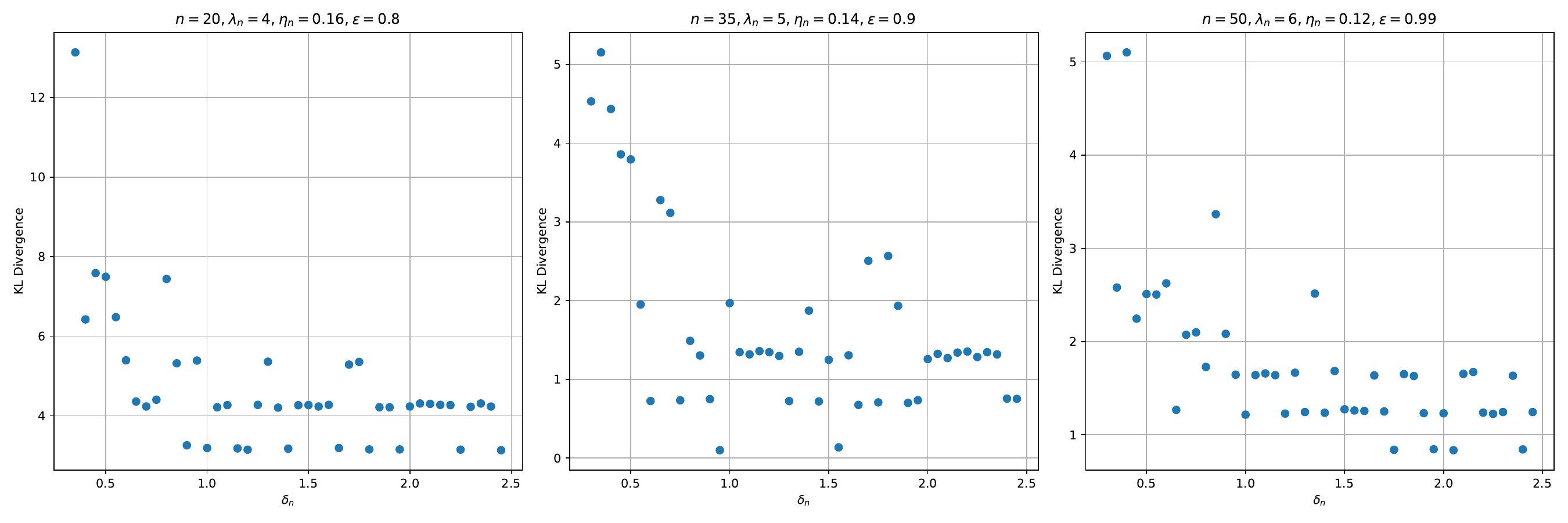}
        \caption{$(\sigma^{2}, \phi, \kappa)$ = $(1,0.1, 0.5)$ for n = $20\;(left),35\;(middle),50\;(right)$.}
        \label{fig:kl05}
    \end{subfigure}

    \begin{subfigure}[b]{0.8\linewidth}
        \centering
        \includegraphics[width=\linewidth]{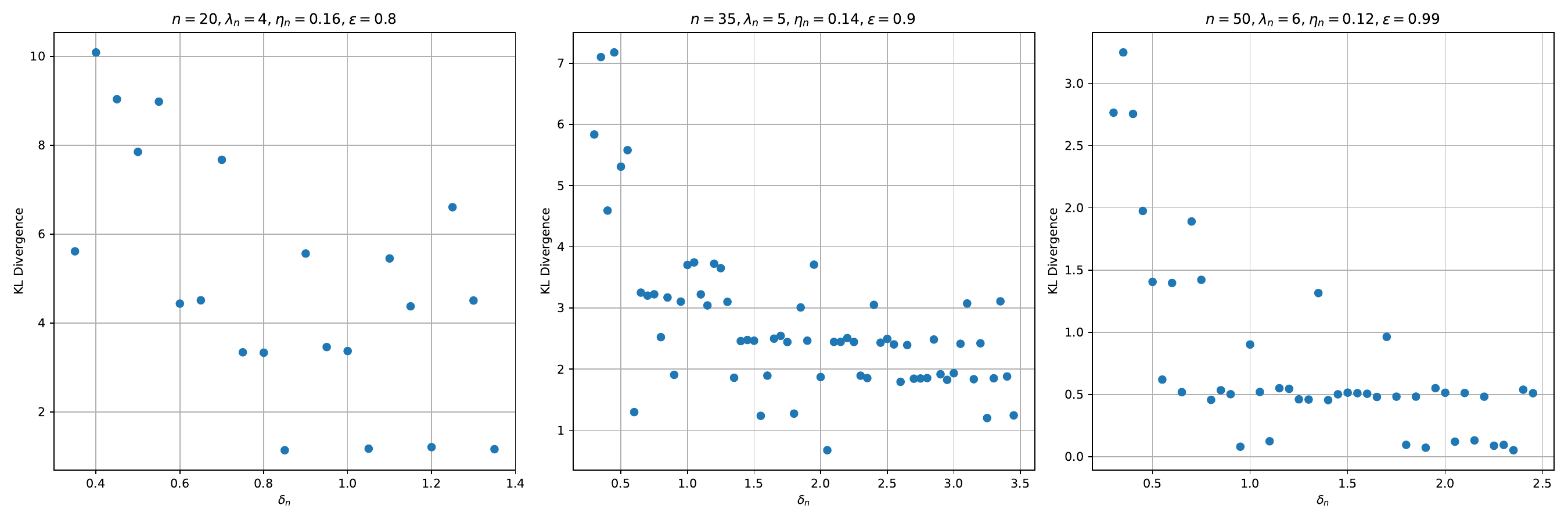}
        \caption{$(\sigma^{2}, \phi, \kappa)$ = $(1,0.1, 1)$ for n = $20\;(left),35\;(middle),50\;(right)$.}
        \label{fig:kl1}
    \end{subfigure}

    \begin{subfigure}[b]{0.8\linewidth}
        \centering
        \includegraphics[width=\linewidth]{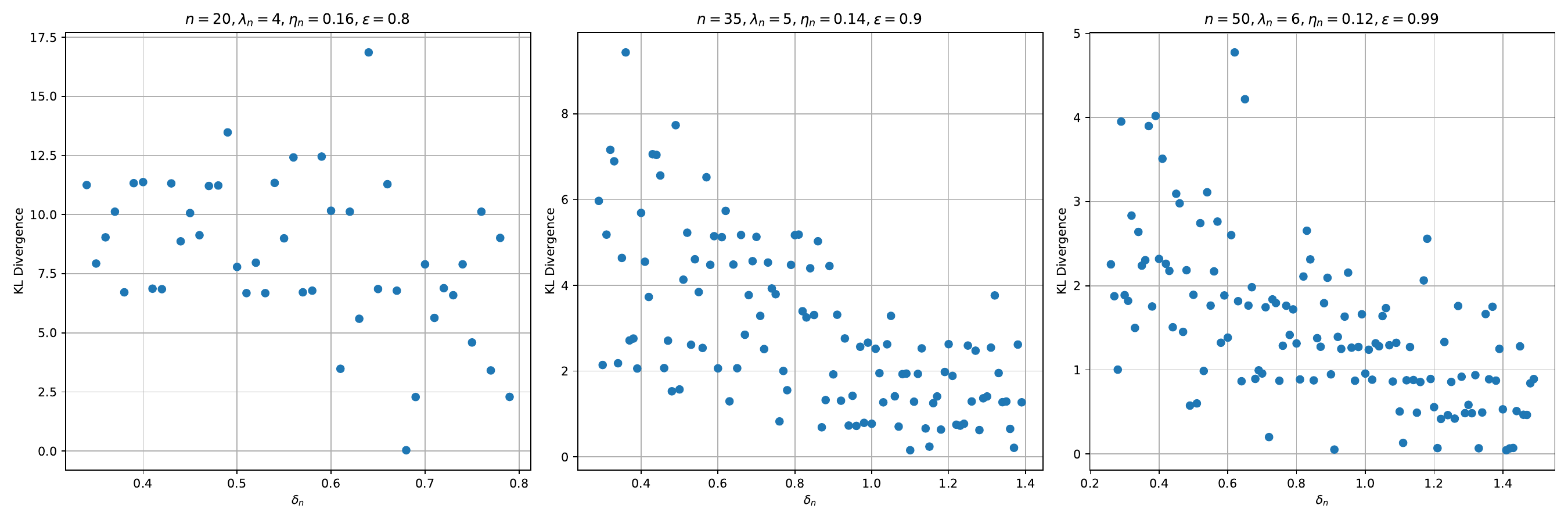}
        \caption{$(\sigma^{2}, \phi, \kappa)$ = $(1,0.1, 1.5)$ for n = $20\;(left),35\;(middle),50\;(right)$.}
        \label{fig:kl15}
    \end{subfigure}

    \begin{subfigure}[b]{0.8\linewidth}
        \centering
        \includegraphics[width=\linewidth]{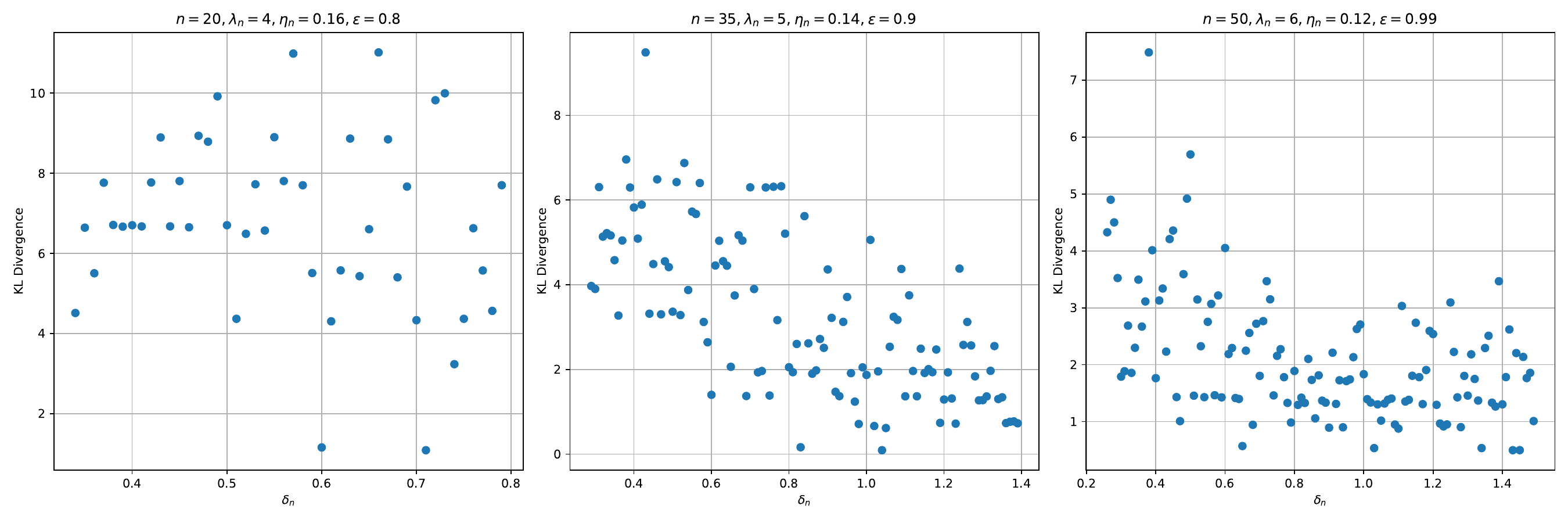}
        \caption{$(\sigma^{2}, \phi, \kappa)$ = $(1,0.1, 2)$ for n = $20\;(left),35\;(middle),50\;(right)$.}
        \label{fig:kl2}
    \end{subfigure}
    
    \caption{Asymptotic convergence of KL divergence with increasing $\delta_n$.}
    \label{fig:kl_snn}
\end{figure}

\begin{figure}[H]
    \centering
    \includegraphics[width=\linewidth]{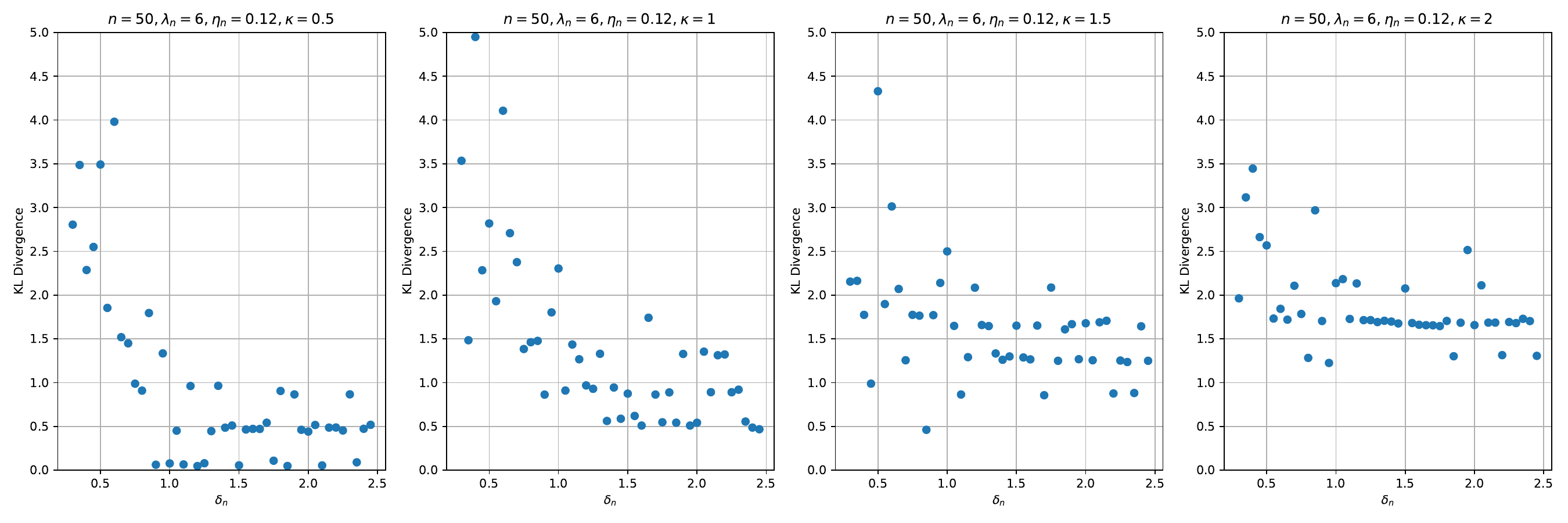}
    \caption{Variation of KL Divergence with $\kappa$ values.}
    \label{fig:kl_kappa}
\end{figure}

\subsection{Consistency for Bounded Spatial Domain: Satellite Image}
In this section, we discuss the consistency behavior, where we study whether the 2-DNN is consistent. We have already empirically validated that the consistency of localized 2-DNN is achieved under increasing $\delta_n, \lambda_n$ and decreasing $\eta_n$. But if we do $\lambda_n$ is fixed then with increasing density in the neighborhood we achieve the convergence of localized 2-DNN functional. We collect the monthly average surface temperature data for major cities\footnotemark[3] in the United States of America (USA) and satellite images from the website\footnotemark[4] where we collect monthly average surface temperature data (SKT)\footnotemark[4], monthly average of air temperature at 2m above from the surface (t2m)\footnotemark[4], and that of 2m above dew point temperature (d2m)\footnotemark[2].\\ 
\footnotetext[3]{\url{https://cds.climate.copernicus.eu}}
\footnotetext[4]{\url{https://power.larc.nasa.gov/data-access-viewer/}}
\begin{figure}[H]
    \centering
    \includegraphics[width=\textwidth]{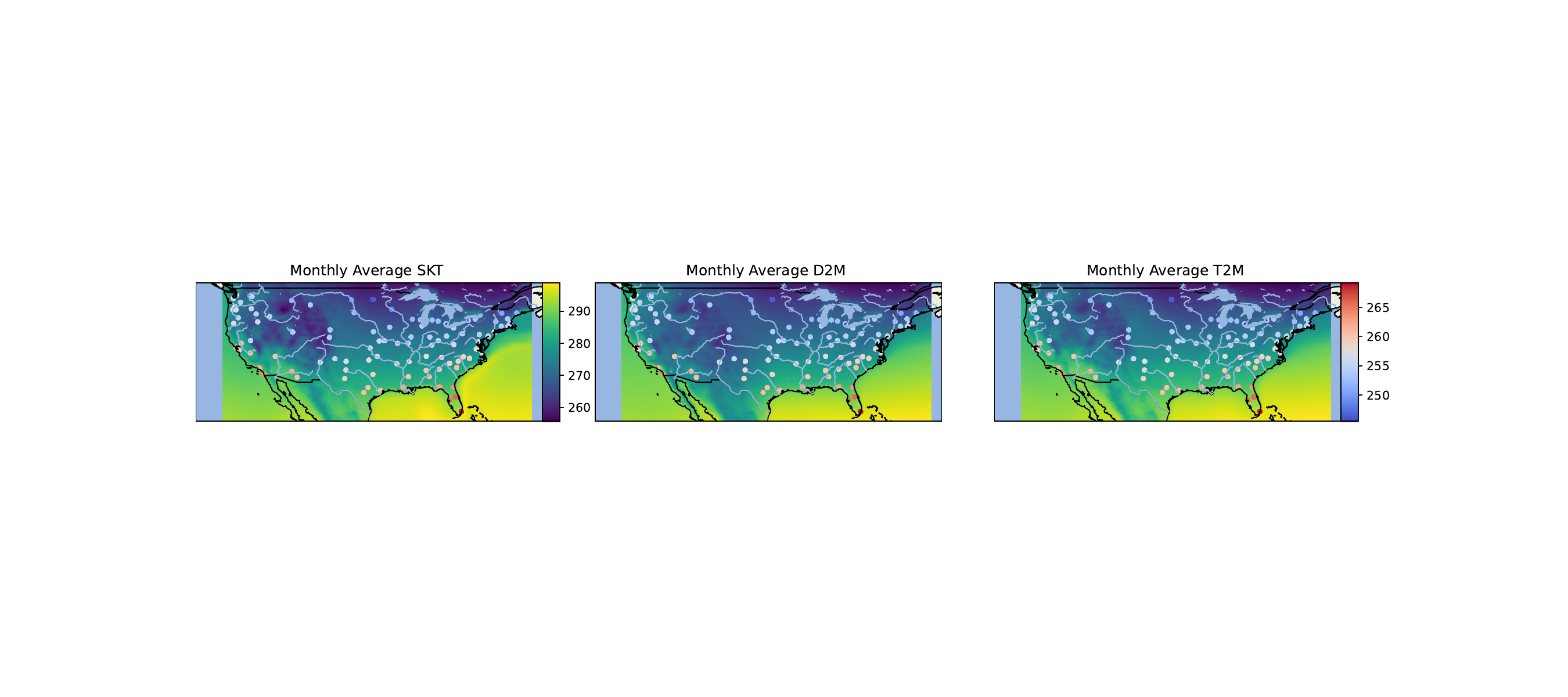}
    \caption{Satellite images with pixel spacing, $\eta_n = 0.25$ and the dotted points are indicating monthly (January) average of 1982 of SKT of cities in USA.}
    \label{fig:us_city}
\end{figure}
Now for each city in Figure~\ref{fig:us_city}, we create a neighborhood of $\delta_n$ and consider those spatial pixels included in this neighborhood. We assume monthly average surface temperature as $y_{city}(\mathbf{s}_i)$ and that of neighborhood pixels are $\{y_{im}(\mathbf{s}_{i_1}), y_{im}(\mathbf{s}_{i_2}), \cdot \cdot \cdot, y_{im}(\mathbf{s}_{i_m})$. Consider t2m as $\mathcal{X}_1$ and d2m as $\mathcal{X}_2$. Also, assume the covariates in those locations included in the neighborhood. Here the initial grid spacing $\eta_0 = 0.25$ and in each resolution to get finer resolution at $l^{th}$ level we consider $\eta_l = \eta_0/ 2^{l-1}$. Thus the grid spacing becomes a decreasing sequence and sampling region $\mathcal{R}_n$ becomes heavily dense. We consider the neural network as 
\begin{equation*}
    \begin{split}
      y_{city}(\mathbf{s}_i) = f_{DNN}\left\{y_{im}(\mathbf{s}_{i_1}), y_{im}(\mathbf{s}_{i_2}), \cdot \cdot \cdot, y_{im}(\mathbf{s}_{i_m})\right.\\
      \left.\mathcal{X}_{1,im}(\mathbf{s}_{i_1}), \mathcal{X}_{1,im}(\mathbf{s}_{i_2}), \cdot \cdot \cdot, \mathcal{X}_{1,im}(\mathbf{s}_{i_m}) \right.\\
      \left.\mathcal{X}_{2,im}(\mathbf{s}_{i_1}), \mathcal{X}_{2,im}(\mathbf{s}_{i_2}), \cdot \cdot \cdot, \mathcal{X}_{2,im}(\mathbf{s}_{i_m}) \right\} + \epsilon_i.
    \end{split}
\end{equation*}

With increasing $\delta_n$ we increase the neighborhood size $m$. To avoid a high-dimensional setup we restrict our number of covariates $(p+1)\Gamma_n = 3m < <  N_n$. The number of times stands from January 1982 to December 2022. In Figure~\ref{fig:us_city} we depict the monthly average of January 1982 of major cities in US\footnotemark[4] and the corresponding satellite image\footnotemark[3]. We use a two-layer deep neural network with \texttt{tanh} activation function in this setup. From Figure~\ref{fig:consist_sat} we can visualize that root mean square error (RMSE) is decreasing with increasing $\delta_n$. We will discuss one important pattern that is detectable from this Figure~\ref{fig:consist_sat} that for $\eta_n = 0.25, \delta_n \geq 0.48$ the RMSE $\leq 9.8$, for $\eta_n = 0.125, \delta_n \geq 0.3$ the RMSE $\leq 9.8$, for $\eta_n = 0.06, \delta_n \geq 0.25$ the RMSE $\leq 9.8$, and for $\eta_n = 0.03, \delta_n \geq 0.18$ the RMSE $\leq 9.8$. Thus we state that whenever we decrease $\eta_n$ our $\Gamma_n \equiv m(n)$ becomes large so that for smaller $\delta_n$ we can achieve the desired accuracy.

\begin{figure}[H]
    \centering
    \includegraphics[width=0.8\linewidth]{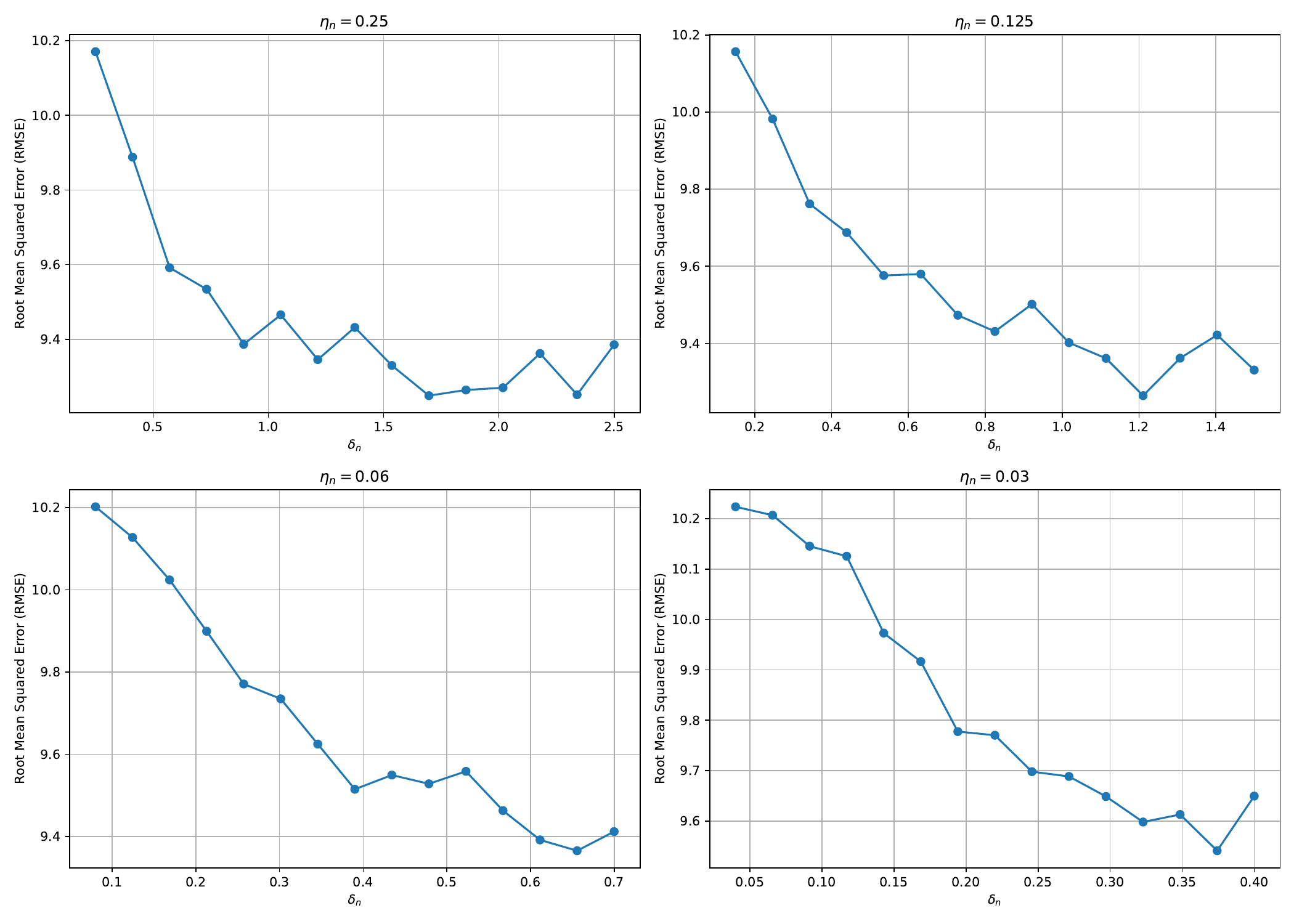}
    \caption{Consistancy of Localized 2-DNN Estimator.}
    \label{fig:consist_sat}
\end{figure}

\section{Conclusion}
From this entire discussion, we have empirically and theoretically proved that localized 2-DNN estimator is consistent with decreasing $\eta_n$, increasing $\Gamma_n$ (for higher resolution) for bounded ($\lambda_n$ is fixed) and unbounded ($\lambda_n = o(n)$) for fixed sampling design of lattice data under mixed increasing spatial sampling. This localized 2-DNN function generalizes spatial regression for the first time. The convergence rate of our localized 2-DNN is faster than the convergence rate of spatial GNN of \cite{zhan2024neural}, \cite{shen2023asymptotic}. We empirically observe the convergence of KL divergence of empirical CDF for observed and predicted data towards zero with decreasing smoothness of the Matern variogram. From the application perspective, we illustrate how, if we collect data from different sources, our localized 2-DNN can iteratively achieve its optimum accuracy level. Although our asymptotic analysis is providing an effective outcome for 2-DNN some limitations direct our future research such as we were desperate here to avoid a high-dimensional setup and therefore we always restrict our neighborhood size. In the functional constraint, we ignore the temporal dependence and assume iid input nodes. We do not theoretically justify how variogram parameters are influencing 2-DNN estimates. These problems are still open problems. We will explore these problems in the coming future.

\section*{Declarations}


\begin{itemize}
\item Conflict of interest/Competing interests: No conflict of interest exists.
\item Data availability: Data is open source, and the available URLs are mentioned in the article.
\item Code availability: Code is available in \url{https://github.com/debjoythakur/Spatial_subsampling_NN}.
\end{itemize}







\begin{appendices}
 \section{Proof of Theorem~\ref{thm:existance}}\label{sec:pf_ex}
\begin{proof} We discuss the proof of the existence of 2-DNN in several steps.
    \begin{enumerate}
        \item[Step-1:] $\mathbb{L}_{n}$ is measurable since,
        \begin{equation*}
            \begin{split}
                &\mathbb{L}_{n} (f) = \frac{1}{n} \sum_{t=1}^{n} 
                \left\{Y_t(\mathbf{s}_i) - f\left(Y_t(\mathbf{s}_{\sigma_1}), Y_t(\mathbf{s}_{\sigma_2}),\cdot \cdot \cdot, Y_t(\mathbf{s}_{\sigma_{\Gamma_n}}),
                \mathcal{X}_t(\mathbf{s}_{\sigma_{1}}),\right.\right.\\
                &\left.\left.
                \mathcal{X}_t(\mathbf{s}_{\sigma_{2}}), \cdot \cdot \cdot, \mathcal{X}_t(\mathbf{s}_{\sigma_{\Gamma_n}}\right) \right\}^{2}\\
                &=\frac{1}{n} \sum_{t=1}^{n} \left[f_0\left(Y_t(\mathbf{s}_{\sigma_1}), Y_t(\mathbf{s}_{\sigma_2}),\cdot \cdot \cdot, Y_t(\mathbf{s}_{\sigma_{\Gamma_n}}),
                \mathcal{X}_t(\mathbf{s}_{\sigma_{1}}),
                \mathcal{X}_t(\mathbf{s}_{\sigma_{2}}), \cdot \cdot \cdot, \mathcal{X}_t(\mathbf{s}_{\sigma_{\Gamma_n}}\right)\right.\\
                &\left.- f\left(Y_t(\mathbf{s}_{\sigma_1}), Y_t(\mathbf{s}_{\sigma_2}),\cdot \cdot \cdot, Y_t(\mathbf{s}_{\sigma_{\Gamma_n}}),
                \mathcal{X}_t(\mathbf{s}_{\sigma_{1}}),
                \mathcal{X}_t(\mathbf{s}_{\sigma_{2}}), \cdot \cdot \cdot, \mathcal{X}_t(\mathbf{s}_{\sigma_{\Gamma_n}}\right) + \epsilon_t\right]^{2}
            \end{split}
        \end{equation*}
        \begin{equation*}
        \begin{split}
            &=\frac{2}{n} \sum_{t=1}^{n} \epsilon_t \left(f_{0t} -f_t\right) + \frac{1}{n} \sum_{t=1}^{n} \epsilon_{t}^{2}+\frac{1}{n} \sum_{t=1}^{n}\left(f_{0t} -f_t\right)^{2},
        \end{split}
        \end{equation*}
        and $\epsilon$ is measurable.
        \item[Step-2:] For fixed $\omega \in \Omega$ we know that $\mathbb{L}_{n}(f) \equiv \mathbb{L}_{n}(f(\omega))$ is continuous in $f$.
        \item[Step-3:]$\mathcal{F}_{r,n}$ is a compact subset of $\mathcal{F}$\\
        $\Leftrightarrow$ $\exists\;\; h:\Theta_{n} \times \mathcal{X} \to \mathcal{F}_{r,n}$ is continuous since $\Theta_{n}$ is a compact subset of $\mathbb{R}^{(r^{2}(d+2) + 2r + 1)}$ and $\mathcal{X}$ is a compact subset of $\mathbb{R}^{d}$
        
        \begin{equation*}
            \begin{split}
                \Theta_{n} &= \left[\vartheta_{0},\cdot \cdot \cdot, \vartheta_{r}, \nu_{00},\cdot \cdot \cdot,\nu_{rr}, \upsilon_{1100},...,\upsilon_{rr(p+1)\Gamma_n} \right]^{'}\\
                &\in [-V_1, V_1]^{r+1} \times [-V_2, V_2]^{(r+1)^2} \times [-V_1, V_1]^{r^{2} (p+2)(\Gamma_n+1)}\\
                \text{where,}\;h(\theta_{n}) &= \vartheta_0 + \sum_{i=1}^{r} \vartheta_i \texttt{tanh}\Bigg(\sum_{j=1}^{r} \nu_{ij}\texttt{tanh}\Bigg(\sum_{k=1}^{p}\sum_{l=1}^{\Gamma_n} \upsilon_{ijkl} X_k(\mathbf{s}_{\sigma_l}) + \\
        &\sum_{l=1}^{\Gamma_n}\upsilon_{ij(p+1)l} Y(\mathbf{s}_{\sigma_l}) + \upsilon_{ij00}\Bigg) + \nu_{i0}\Bigg)
            \end{split}
        \end{equation*}
        $\Leftrightarrow$ It's enough to show for all\; $\theta_{1n}, \theta_{2n} \in \Theta_{n}$
        \[
        \lvert \lvert h(\theta_{1n}) - h(\theta_{2n}) \rvert \rvert_{n}^{2} \leq V^{*} \lvert \lvert \theta_{1n} - \theta_{2n} \rvert \rvert_{n}^{2}
        \]
        Therefore our main responsibility is to find $V^{*}$. \\
        For ease of notation assume,
        \[
        \mathcal{Y}_t= \left[X_{1t}(\mathbf{s}_{\sigma_1}), \cdot \cdot \cdot, X_{1t}(\mathbf{s}_{\sigma_{\Gamma_n}}),\cdot \cdot \cdot, X_{pt}(\mathbf{s}_{\sigma_{\Gamma_n}}), Y_t(\mathbf{s}_{\sigma_{1}}),\cdot \cdot\cdot, Y_t(\mathbf{s}_{\sigma_{\Gamma_n}})\right]
        \]
        \begin{equation*}
    \begin{split}
        &\lvert \lvert h(\theta_{1n}) - h(\theta_{2n}) \rvert \rvert_{n}^{2}\\
        &\leq \frac{1}{n} \sum_{t=1}^{n} \left[\left\lvert \vartheta_{0}^{(1)} - \vartheta_{0}^{(2)} \right\rvert + \sum_{i=1}^{r} \left\lvert \vartheta_i^{(1)} \texttt{tanh}\left(\sum_{j=1}^{r} \nu_{ij}^{(1)}\texttt{tanh}\left(\sum_{k=1}^{(p+1)\Gamma_n} \upsilon_{ijk}^{(1)} \mathcal{Y}_{k,t} + \upsilon_{ij0}^{(1)}\right) + \nu_{i0}^{(1)} \right)\right.\right. \\
        &\hspace{3cm} \left.\left.- \vartheta_i^{(2)} \texttt{tanh}\left(\sum_{j=1}^{r} \nu_{ij}^{(2)}\texttt{tanh}\left(\sum_{k=1}^{(p+1)\Gamma_n} \upsilon_{ijk}^{(2)} \mathcal{Y}_{k,t} + \upsilon_{ij0}^{(2)}\right) + \nu_{i0}^{(2)} \right) \right\rvert \right]^{2} 
    \end{split}
\end{equation*}
\begin{equation*}
 \begin{split}
        &\lvert \lvert h(\theta_{1n}) - h(\theta_{2n}) \rvert \rvert_{n}^{2}\\
        &\leq \frac{1}{n} \sum_{t=1}^{n} \left[
            \sum_{i=0}^{r} \left\lvert \vartheta_i^{(1)} - \vartheta_i^{(2)} \right\rvert 
            +V_1 \sum_{i=1}^{r} \left\lvert \nu_{i0}^{(1)} - \nu_{i0}^{(2)} \right\rvert \right.\\
            &\left. + V_1 \sum_{i=1}^{r} \sum_{j=1}^{r} \left\lvert
                \nu_{ij}^{(1)} \texttt{tanh}\left( \sum_{k=1}^{(p+1)\Gamma_n} \upsilon_{ijk}^{(1)} \mathcal{Y}_{k,t} + \upsilon_{ij0}^{(1)} \right) 
                - \nu_{ij}^{(2)} \texttt{tanh}\left( \sum_{k=1}^{(p+1)\Gamma_n} \upsilon_{ijk}^{(2)} \mathcal{Y}_{k,t} + \upsilon_{ij0}^{(2)} \right)
            \right\rvert\right]^{2}\\
        & \leq \frac{1}{n} \sum_{t=1}^{n} \left[
            \sum_{i=0}^{r} \left\lvert \vartheta_i^{(1)} - \vartheta_i^{(2)} \right\rvert 
            +V_1 \sum_{i=1}^{r}\sum_{j=0}^{r} \left\lvert \nu_{ij}^{(1)} - \nu_{ij}^{(2)} \right\rvert  \right.\\
            & \left. + V_1 V_2 \sum_{i=1}^{r} \sum_{j=1}^{r} \left \lvert \upsilon_{ij0}^{(1)} - \upsilon_{ij0}^{(2)} \right \rvert 
            +V_1 V_2 \sum_{i=1}^{r} \sum_{j=1}^{r} \sum_{k=1}^{(p+1)\Gamma_n} \left\lvert \upsilon_{ijk}^{(1)} \mathcal{Y}_{k,t} - \upsilon_{ijk}^{(2)} \mathcal{Y}_{k,t} \right \rvert 
            \right]^{2}\\
        & \leq \frac{1}{n} \sum_{t=1}^{n} \left[
            \sum_{i=0}^{r} \left\lvert \vartheta_i^{(1)} - \vartheta_i^{(2)} \right\rvert 
            +V_1 \sum_{i=0}^{r}\sum_{j=0}^{r} \left\lvert \nu_{ij}^{(1)} - \nu_{ij}^{(2)} \right\rvert
            +V_1 V_2 \sum_{i=0}^{r} \sum_{j=0}^{r} \sum_{k=0}^{(p+1)\Gamma_n} \left\lvert \upsilon_{ijk}^{(1)}- \upsilon_{ijk}^{(2)}\right \rvert 
            \right]^{2}  \\
        & \leq  \left[\sum_{i=0}^{r}\sum_{j=0}^{r} \sum_{k=0}^{(p+1)\Gamma_n} \left\lvert \vartheta_i^{(1)} - \vartheta_i^{(2)} \right\rvert + V_1 \sum_{i=0}^{r}\sum_{j=0}^{r}\sum_{k=0}^{(p+1)\Gamma_n} \left\lvert \nu_{ij}^{(1)} - \nu_{ij}^{(2)} \right\rvert\right.\\
        &\left.+V_1 V_2 \sum_{i=0}^{r}\sum_{j=0}^{r}\sum_{k=0}^{(p+1)\Gamma_n} \left\lvert \upsilon_{ijk}^{(1)} - \upsilon_{ijk}^{(2)} \right\rvert \right]^{2} \\
        & \leq  \left(V_1 V_2\left\{(r+1)(p+1)\Gamma_n + (p+1)\Gamma_n\right\} \right)^{2} \lvert \lvert \theta_{1n} - \theta_{2n} \rvert \rvert_{n}^{2}.
    \end{split}
\end{equation*}
As a result, we get
        \begin{equation}
            V^{*} = \Bigg(V_1 V_2\left\{(r+1)(p+1)\Gamma_n + (p+1)\Gamma_n\right\} \Bigg)^{2}.
        \end{equation}
    \end{enumerate}
\end{proof}	
\section{Proof of Theorem~\ref{thm:convergence}}\label{sec:pf_conv}
\begin{proof}
    According to Theorem 3.1 in \cite{chen2007large} and Corollary 2.6 in \cite{white1991some} it's enough to show that 
    \[
    \text{sup}_{f \in \mathcal{F}_{r,n}} \lvert \mathbb{L}_{n}(f) - \mathcal{L}_{n}(f) \rvert \to^{P^{*}} 0\;\; \text{as}\;\; n\to \infty
    \]
    Now we will give the outline of the proof
    \begin{enumerate}
        \item[Step-1:]For any $\varepsilon > 0$ we have
        \begin{equation}
            \begin{split}
                &P^{*}\left( \text{sup}_{f \in \mathcal{F}_{r,n}} \lvert \mathbb{L}_{n}(f) - \mathcal{L}_{n}(f) \rvert > \frac{\varepsilon}{2}\right)
                \leq P\left(\left\lvert \frac{1}{n} \sum_{t=1}^{n} \varepsilon_{t}^{2} - \sigma^{4} \right\rvert > \frac{\varepsilon}{4}\right) \\
                &+ P^{*}\left( \text{Sup}_{f \in \mathcal{F}_{r,n}} \left\lvert\frac{1}{n} \sum_{t=1}^{n} \epsilon_{i} \left\{ f(\mathcal{Y}_t) - f_{0}(\mathcal{Y}_t)\right\} \right\rvert > \frac{\varepsilon}{4}\right)
            \end{split}
        \end{equation}
        \item[Step-2:] The First part will converge to zero using WLLN but for the second part using Markov's Inequality we are required to prove 
        \[
         E_{\epsilon}\left( \text{Sup}_{f \in \mathcal{F}_{r,n}} \left\lvert\frac{1}{n} \sum_{t=1}^{n} \epsilon_{t} \left\{ f(\mathcal{Y}_t) - f_{0}(\mathcal{Y}_t)\right\} \right\rvert\right) \to 0\; \text{as} \; n\to \infty
        \]
        \item[Step-3:] According to Symmetrisation for $\alpha-$Mixing spatial surface with grid spacing $\eta_{n}  \to 0$ as $n \to \infty$  
        \begin{equation}
            \begin{split}
                E_{\epsilon} \lvert \lvert \mathbb{P}_{n} - P \rvert \rvert_{\mathcal{F}_{r,n}} 
 = E_{\epsilon}\left( \text{Sup}_{f \in \mathcal{F}_{r,n}} \left\lvert\frac{1}{n} \sum_{t=1}^{n} \epsilon_{t} \left\{ f(\mathcal{Y}_t) - f_{0}(\mathcal{Y}_t)\right\} \right\rvert\right)\\
 \leq 2 E_{\epsilon}E_{\rho} \text{Sup}_{f \in \mathcal{F}_{r,n}} \left\lvert\frac{1}{n} \sum_{t=1}^{n} \rho_t\epsilon_{t} \left\{ f(\mathcal{Y}_t) - f_{0}(\mathcal{Y}_t)\right\} \right\rvert
            \end{split}
        \end{equation}
      where $\rho_t;\;t=1,2,...,n$ are Rademacher random variable independent of $\epsilon_{1}, \epsilon_{2},..., \epsilon_{n}$. For fixed values of $\epsilon_{1}, \epsilon_{2},..., \epsilon_{n}$ (conditionally) $\sum_{t=1}^{n} \rho_t\epsilon_{t} \left\{ f(\mathcal{Y}_t) - f_{0}(\mathcal{Y}_t)\right\}$ is a sub-gaussian process.
      \item[Step-4:] According to the definition of covering number and Corollary 2.2.8 of \cite{wellner2013weak} we will get a universal constant $K$. Then using Dudley's integral entropy bound in VC-dimension $\nu$ with any $f_{n}^{*} \in \mathcal{F}_{r,n}$ with $n \geq N_1$ we will get 
      \begin{equation*}
          \begin{split}
            &E_{\rho} \text{Sup}_{f \in \mathcal{F}_{r,n}} \left\lvert\frac{1}{n} \sum_{t=1}^{n} \rho_t\epsilon_{i} \left\{ f(\mathcal{Y}_t) - f_{0}(\mathcal{Y}_t)\right\} \right\rvert\\
            &\leq   E_{\rho}  \left\lvert\frac{1}{n} \sum_{t=1}^{n} \rho_t\epsilon_{t} \left\{ f_{n}^{*}(\mathcal{Y}_t) - f_{0}(\mathcal{Y}_t)\right\} \right\rvert \\
            &+K\int_{0}^{\infty}\sqrt{\frac{log\;D(\varphi, \mathcal{F}_{r,n}, \lvert \lvert . \rvert \rvert_{\infty})}{n}}d\varphi  
          \end{split}
      \end{equation*}
      \begin{equation*}
          \begin{split}
               &\leq   E_{\rho}  \left\lvert\frac{1}{n} \sum_{t=1}^{n} \rho_t\epsilon_{t} \left\{ f_{n}^{*}(\mathcal{Y}_t) - f_{0}(\mathcal{Y}_t)\right\} \right\rvert \\
        &+K\sqrt{\frac{\nu}{n}}\int_{0}^{2V_1V_2}\sqrt{\frac{log\;\mathcal{N}(\varphi/2, \mathcal{F}_{r,n}, \lvert \lvert . \rvert \rvert_{\infty})}{n}} d\varphi
          \end{split}
      \end{equation*}
      \begin{equation}\label{eq:int_step_5}
          \begin{split}
               &\leq   E_{\rho}  \left\lvert\frac{1}{n} \sum_{t=1}^{n} \rho_t\epsilon_{t} \left\{ f_{n}^{*}(\mathcal{Y}_t) - f_{0}(\mathcal{Y}_t)\right\} \right\rvert \\
               &+K\sqrt{\frac{\nu}{n}}\int_{0}^{2V_1V_2}\sqrt{\frac{log\;\mathcal{N}\left(\frac{\varphi}{2V_1V_2\sqrt{\sigma^{2} + 1}}, \mathcal{F}_{r,n}, \lvert \lvert . \rvert \rvert_{\infty}\right)}{n} d\xi}\\
        &= I_{3} + I_{4}
          \end{split}
      \end{equation}
      \item[Step-5:] According to \cite{hornik1989multilayer} we utilize the concept of universal approximation of $I_{3}$ in (\ref{eq:int_step_5}) and then it will be 
      \begin{equation}
          \begin{split}
         &2\;E_{\epsilon}E_{\rho}  \left\lvert\frac{1}{n} \sum_{t=1}^{n} \rho_t\epsilon_{t} \left\{ f_{n}^{*}(\mathcal{Y}_t) - f_{0}(\mathcal{Y}_t)\right\} \right\rvert \\
        &\leq 2 \sqrt{\sigma^{2} + 1} \times \text{Sup}_{\mathcal{Y}} \left\lvert f_{n}^{*}(\mathcal{Y}_t) - f_{0}(\mathcal{Y}_t) \right\rvert\\
        &\leq 2 \sqrt{\sigma^{2} + 1} \times \text{Sup}_{\mathcal{Y}} \left\lvert f_{n}^{*}(\mathcal{Y}_t) - f_{0}(\mathcal{Y}_t) \right\rvert\to 0 \;\text{as} \; n\to \infty
          \end{split}
      \end{equation}
      \item[Step-6:] According to \cite{anthony1999neural} from Lemma 14.3, Lemma 14.7
      \begin{equation*}
      \begin{split}
          &\mathcal{N}\left(\frac{\varphi}{2V_1V_2\sqrt{\sigma^{2} + 1}}, \mathcal{F}_{r,n}, \lvert \lvert . \rvert \rvert_{\infty}\rvert\right)\\
          &\leq \text{max}_{\mathcal{Y}} \mathcal{N}\left(\frac{\varphi}{V_1 \sqrt{\sigma^{2} + 1}}, \mathcal{F}^{(1)}_{r,n}, \lvert \lvert . \rvert \rvert^{\xi}_{\infty})\right)\\
          &\times \mathcal{N}\left(\frac{\varphi}{V_1V_2\sqrt{\sigma^{2} + 1}}, \mathcal{F}^{(2)}_{r,n}, \lvert \lvert . \rvert \rvert_{\infty}\right)\\
      \end{split}
      \end{equation*}
      \begin{equation*}
          \begin{split}
               &\leq \left(\frac{16e V_2 \sqrt{\sigma^{2} + 1}}{\varphi}\right)^{r^{2} }
           \times \left(\frac{(V_1 V_2)^{2} \left\{r((p+1)\Gamma_n+4)+2\right\}}{2(V_1 V_2 - 1)\varphi} \right)^{r((p+1)\Gamma_n+4) + 2}\\
           &\leq \Tilde{C}_{r, p,\Gamma_n,V_1, V_2}\; \varphi^{-(r^{2} + r((p+1)\Gamma_n+4) + 2)}\\
          \end{split}
      \end{equation*}
      \begin{equation*}
      \begin{split}
           &\Leftrightarrow log\;\mathcal{N}\left(\frac{\varphi}{2V_1V_2\sqrt{\sigma^{2} + 1}}, \mathcal{F}_{r,n}, \lvert \lvert . \rvert \rvert_{\infty}\rvert\right) \leq
           r^{2}\log\left(\frac{16e V_2 \sqrt{\sigma^{2} +1}}{\varphi}\right)\\ 
           &+ \left(r((p+1)\Gamma_n+4)+2\right)\log\left(\frac{(V_1 V_2)^{2} (r((p+1)\Gamma_n +4)+2) }{2 (V_1 V_2 - 1) \varphi}\right)^{r((p+1)\Gamma_n + 4)+2}
      \end{split}
      \end{equation*}
      According to Assumption~(A2.1) we deduce,
      \[
      \left(8 V_2 \sqrt{\sigma^{2} + 1}\right) (r(((p+1)\Gamma_n)+4) + 2) \leq \left(r^{2} + r(((p+1)\Gamma_n)+4) + 2\right)
      \]
      then,
      \begin{equation*}
          \begin{split}
              &\log\left(\frac{\left(8 V_2 \sqrt{(\sigma^{2} + 1)}\right) (r(((p+1)\Gamma_n)+4) + 2) (V_1 V_2)^{2}}{V_1V_2 - 1}\right)^{\left(r^{2} + r(((p+1)\Gamma_n)+4) + 2\right)}\\
            &+\left(r^{2} + r(((p+1)\Gamma_n)+4) + 2\right) \log (1/\varphi)\\
         &\leq \left(r^{2} + r(((p+1)\Gamma_n)+4) + 2\right)\log\left(\frac{ (r^{2} + r(((p+1)\Gamma_n)+4) + 2) (V_1V_2)^{2}}{V_1V_2 - 1}\right)\\
         &+ \left(r^{2} + r(((p+1)\Gamma_n)+4) + 2\right) \log (1/\varphi)\\
         &\leq \left(r^{2} + r(((p+1)\Gamma_n)+4) + 2\right)\log\left(\frac{ (r^{2} + r(((p+1)\Gamma_n)+4) + 2) (V_1 V_2)^{2}}{V_1 V_2 - 1}\right)\\
         &\times \left(1+\frac{1}{\varphi}\right).
          \end{split}
      \end{equation*}
      Consider $C_{r, d,V_1, V_2} = \left(r^{2} + r(((p+1)\Gamma_n)+4) + 2\right)\log\left(\frac{ (r^{2} + r(((p+1)\Gamma_n)+4) + 2) (V_1 V_2)^{2}}{V_1 V_2 - 1}\right)$. Hence we have got 
      \begin{equation*}
          \begin{split}
               &\int_{0}^{2V_1V_2} \sqrt{log\;\mathcal{N}\left(\frac{\varphi}{2V_1V_2\sqrt{\sigma^{2} + 1}}, \mathcal{F}_{r,n}, \lvert \lvert . \rvert \rvert_{\infty}\rvert\right)}\\
               &\leq  \int_{0}^{2V_1V_2} C^{1/2}_{r, p, \Gamma_n,V_1, V_2}\left(1+\frac{1}{\varphi}\right)^{1/2} d\varphi\\
               &\leq 4\sqrt{2} C^{1/2}_{r, p, \Gamma_n,V_1, V_2}
          \end{split}
      \end{equation*}
    \end{enumerate}
 \item[Step-7:]Combining the previous all the steps, Step-3, Step-4, Step-5, and Step-6 in Step-2 we will get
 \begin{equation}\label{eq:conv_final}
     \begin{split}
         &E_{\epsilon}\left( \text{Sup}_{f \in \mathcal{F}_{r,n}} \left\lvert\frac{1}{n} \sum_{t=1}^{n} \epsilon_{t} \left\{ f(\mathcal{Y}_t) - f_{0}(\mathcal{Y}_t)\right\} \right\rvert\right)\\
          &\leq 2 E_{\epsilon}E_{\rho} \text{Sup}_{f \in \mathcal{F}_{r,n}} \left\lvert\frac{1}{n} \sum_{t=1}^{n} \rho_t\epsilon_{t} \left\{ f(\mathcal{Y}_t) - f_{0}(\mathcal{Y}_t)\right\} \right\rvert\\
            & \leq 2 \sqrt{\sigma^{2} + 1} \times \text{Sup}_{\mathcal{Y}} \left\lvert f_{n}^{*}(\mathcal{Y}_t) - f_{0}(\mathcal{Y}_t) \right\rvert\\
          &+2K\sqrt{\frac{\nu}{n}}\int_{0}^{2V_1V_2}\sqrt{\frac{log\;\mathcal{N}\left(\frac{\varphi}{2V_1V_2\sqrt{\sigma^{2} + 1}}, \mathcal{F}_{r,n}, \lvert \lvert . \rvert \rvert_{\infty}\right)}{n} d\varphi}\\
          &\approx 2 \sqrt{\sigma^{2} + 1} \times \text{Sup}_{\mathcal{Y}} \left\lvert f_{n}^{*}(\mathcal{Y}_t) - f_{0}(\mathcal{Y}_t) \right\rvert
          + 8\sqrt{2}K  
          \sqrt{\frac{\nu C_{r,p,\Gamma_n,V_1, V_2}}{n^{2}}}
     \end{split}
 \end{equation}
 
 \begin{equation*}
     \begin{split}
        &= 2 \sqrt{\sigma^{2} + 1} \times \text{Sup}_{\mathcal{Y}} \left\lvert f_{n}^{*}(\mathcal{Y}_t) - f_{0}(\mathcal{Y}_t) \right\rvert\\
        & + 8\sqrt{2\nu}K  
          \sqrt{\frac{\left(r^{2} + r(((p+1)\Gamma_n)+4) + 2\right)(V_1 V_2)^{2}\log\left(\frac{ (r^{2} + r(((p+1)\Gamma_n)+4) + 2) (V_1V_2)^{2}}{V_1V_2 - 1}\right)}{n^{2}}}\\
        &\approx 2 \sqrt{\sigma^{2} + 1} \times \text{Sup}_{\mathcal{Y}} \left\lvert f_{n}^{*}(\mathcal{Y}_t) - f_{0}(\mathcal{Y}_t) \right\rvert
         + 8\sqrt{2}K\sqrt{\frac{\nu}{n}} \\ 
        &\sqrt{\frac{\left(r^{2} + r(((p+1)\Gamma_n)+4) + 2\right)(V_1V_2)^{2}\log\left( (r^{2} + r(((p+1)\Gamma_n)+4) + 2) (V_1V_2)\right)}{n}}
     \end{split}
 \end{equation*}
 To prove the convergence it's necessary to satisfy (\ref{eq:conv_final}) should converge to $0$ as $n \to \infty$. The first part of the addition converges to zero because of the universal approximation theorem. For the second part since we have assumed $\Gamma_n = o(n^{\beta}), r = o(n^{\psi})$ and $\psi + \beta < 1$ then the second part will converge to $0$ as $n \to \infty$. 
 \end{proof}
 \section{Proof of Theorem~\ref{thm:conv_rate}}\label{sec:pf_convrt}
 \begin{proof}
This proof will discuss the convergence rate of spatial 2-DNN in detail.
    \begin{itemize}
        \item[Step-I:] According to \cite{anthony1999neural} from Lemma 14.3, Lemma 14.7 of covering number inequality utilizing exponential decay of $\alpha-$Mixing condition will satisfy the following modified inequality  
        \begin{equation*}
        \begin{split}
            \mathcal{N}_{\infty}\left(\varphi, \mathcal{F}_{r,n}, \lvert \lvert \cdot \rvert \rvert_{\infty}\right) &\leq \left(\frac{2e}{\varphi}\right)^{r^{2}} \times \left(\frac{\left(V_1 V_2\right)^{2} (r ((p+1) \Gamma_n+4) + 2)}{\varphi \left(V_1 V_2 -1\right)}\right)^{r ((p+1)\Gamma_n+4) +2}\\
            \Rightarrow log\mathcal{N}_{\infty}\left(\varphi, \mathcal{F}_{r,n}, \lvert \lvert \cdot \rvert \rvert_{\infty}\right) &\leq r^{2} log \left(\frac{2e}{\varphi}\right) + (r ((p+1)\Gamma_n+4) +2)\\
            &\times log\left(\frac{\left(V_1 V_2\right)^{2} (r ((p+1)\Gamma_n+4) + 2)}{\varphi \left(V_1 V_2 -1\right)}\right)\\
            &\leq (r^{2}+r ((p+1)\Gamma_n+4) +2)\\
            &\times \log\left(\frac{2e\left(V_1V_2\right)^{2} (r^{2}+r ((p+1)\Gamma_n+4) + 2)}{\varphi \left(V_1 V_2 -1\right)}\right)
        \end{split} 
        \end{equation*}
        \item[Step - II] From Lemma 3.8 of \cite{mendelson2003advanced} for $\xi < 1$ along with the we have 
        \begin{equation*}
        \begin{split}
            &\int_{0}^{\xi} \sqrt{log\mathcal{N}_{\infty}\left(\varphi, \mathcal{F}_{r,n}, \lvert \lvert \cdot \rvert \rvert_{\infty}\right)} d\varphi \\
            &\leq \sqrt{(r^{2}+r ((p+1)\Gamma_n+4) +2)} \xi \sqrt{log\left(\frac{2e\left(V_1V_2\right)^{2} (r^{2}+r ((p+1)\Gamma_n+4) + 2)}{\xi \left(V_1 V_2 -1\right)}\right)}\\
            &= \Upsilon_{n}(\xi)
        \end{split}
        \end{equation*}
        \item[Step-III] Define a function $h:\xi \to \frac{\Upsilon_{n}(\xi)}{\xi^a}$ in $(0,\infty)$ for some $a > 0$ such that 
        \begin{equation*}
            \begin{split}
                 h(\xi) &=  \sqrt{(r^{2}+r ((p+1)\Gamma_n+4) +2)} \xi^{(1-a)}\\ &\times \sqrt{log\left(\frac{2e\left(V_1V_2\right)^{2} (r^{2}+r ((p+1)\Gamma_n+4) + 2)}{\xi \left(V_1 V_2 -1\right)}\right)}
            \end{split}
        \end{equation*}
        its derivative will become,
        \begin{equation*}
        \begin{split} 
           h^{'}(\xi) &= \frac{(1-a)}{\xi^{a}} \sqrt{\log\left(\frac{2e\left(V_1 V_2\right)^{2} (r^{2}+r ((p+1)\Gamma_n+4) + 2)}{\xi \left(V_1 V_2 -1\right)}\right)}\\
           &- \frac{\xi^{-a}}{\sqrt{log\left(\frac{2e\left(V_1 V_2\right)^{2} (r^{2} + r ((p+1)\Gamma_n+4) + 2)}{\xi \left(V_1 V_2 -1\right)}\right)}}
        \end{split}
        \end{equation*}
        Thus the function $h$ is decreasing whenever, $a > 1$ and $0 < \xi < 1$.
        \item[Step-IV] Assume $\gamma_n \lesssim \lvert \lvert \pi_{r,n}f_0 - f_0\rvert \rvert_{n}^{-1}$ where $\gamma_n = \sqrt{\zeta_n^{-1}}$
        \begin{equation*}
            \begin{split}
                \gamma_{n}^{2} \Upsilon_{n}\left(\frac{1}{\gamma_{n}}\right) &= \gamma_n \sqrt{\left(r^{2} + r ((p+1)\Gamma_n+4) +2\right)}\\
                & \sqrt{log(\gamma_n) + \log\left(\frac{2e\left(V_1V_2\right)^{2} (r^{2}+r ((p+1)\Gamma_n+4) + 2)}{ \left(V_1 V_2 -1\right)}\right)}\\
                &\lesssim \gamma_n \sqrt{\left(r^{2} + r ((p+1)\Gamma_n+4) +2\right)}\\ &\sqrt{\log(\gamma_n) +2+ log\left(\left(V_1V_2\right)^{2} (r^{2}+r ((p+1)\Gamma_n+4) + 2)\right)}
            \end{split}
        \end{equation*}
        \item[Step-V] Now $\gamma_{n}^{2} \Upsilon_{n}\left(\frac{1}{\gamma_{n}}\right) \lesssim \sqrt{n}$ if and only if 
\begin{equation*}
    \begin{aligned}
        &\gamma_n^{2} \left(r^{2} + r ((p+1)\Gamma_n+4) +2\right)\\ &\left(\log(\gamma_n) + \log\left(\left(V_1V_2\right)^{2} (r^{2}+r ((p+1)\Gamma_n+4) + 2)\right)\log n\right)
        \lesssim n \\
        &\text{Therefore,}\; \gamma_n \lesssim\; \\
        &\text{min}\;\left\{\left(\frac{n}{\left(r^{2} + r ((p+1)\Gamma_n+4) +2\right) \left(\log\left(\left(V_1V_2\right)^{2} (r^{2}+r ((p+1)\Gamma_n+4) + 2)\right)\right)}\right)^{1/2}, \right.\\
        &\left. \left(\frac{n}{\left(r^{2} + r ((p+1)\Gamma_n+4) +2\right) \log(n^{\psi + \beta}\log n)}\right)^{1/2}\right\}.
    \end{aligned}
\end{equation*}
 If we replace $\Gamma_n = o(n^{\psi}), V_1V_2 = o(\sqrt{\log n})$ and $r = o\left(n^{\beta}\right)$ such that $\psi + \beta < 1$ then,
 \begin{equation}\label{eq:final_cr}
     \begin{split}
         \gamma_n \lesssim\;
        &\text{min}\;\left\{\left(\frac{n}{n^{\psi + \beta} \log (n^{\psi + \beta} \log n)}\right)^{1/2},  \left(\frac{n}{n^{\psi + \beta} \log (n\log n )}\right)^{1/2}\right\}.
     \end{split}
 \end{equation}
From (\ref{eq:final_cr}) using Lemma 1, Lemma 2 of \cite{shen2023asymptotic}, Theorem 3.4.1 of \cite{van1996weak} and triangle inequality we get 
\begin{equation*}
        \begin{aligned}
           &\lvert \lvert \hat{f}_n - f_0 \rvert \rvert_{n} \leq \lvert \lvert \hat{f}_n - \pi_{r,n}f_0 \rvert \rvert_{n} + \lvert \lvert \pi_{r,n}f_0 - f_0 \rvert \rvert_{n}\\
            &=\mathcal{O}_{P}\left(max\left\{ \left\lvert\lvert\pi_{r,n} f_{0} - f_{0}\right\rvert\rvert_{n},\right.\right. \\
            &\left. \left. \sqrt{\frac{\left(r^{2} + r ((p+1)\Gamma_n+4) +2\right)\log\left(r^{2} + r ((p+1)\Gamma_n+4) +2\right)(V_1V_2)^{2}}{n}},\right.\right.\\
         &\left.\left. \sqrt{\frac{\left(r^{2} + r ((p+1)\Gamma_n+4) +2\right)\log (n\log n)}{n}} \right\}\right)\\  
          \approx \mathcal{O}_{P}&\left\{\text{max}\;\left[\left\lvert\lvert\pi_{r,n} f_{0} - f_{0}\right\rvert\rvert_{n}, 
         \left(\frac{n}{n^{\psi + \beta} \log (n^{\psi + \beta} \log n)}\right)^{-1/2},\right.\right.\\
         &\left.\left.
         \left(\frac{n}{n^{\psi + \beta} \log (n\log n )}\right)^{-1/2}\right]\right\}.
        \end{aligned}
    \end{equation*}
    \end{itemize}
\end{proof}

\begin{enumerate}
    \item[Step-1:] For any \(\varepsilon > 0\), we have
    \begin{equation}
        \begin{split}
            &P^{*}\left( \sup_{f \in \mathcal{F}_{r,n}} \left| \mathbb{L}_{n,\lambda}(f) - \mathcal{L}_{n,\lambda}(f) \right| > \frac{\varepsilon}{2}\right)
            \leq P\left(\left\lvert \frac{1}{n} \sum_{t=1}^{n} \varepsilon_{t}^{2} - \sigma^{4} \right\rvert > \frac{\varepsilon}{4}\right) \\
            &+ P^{*}\left( \sup_{f \in \mathcal{F}_{r,n}} \left|\frac{1}{n} \sum_{t=1}^{n} \epsilon_{t} \left\{ f(\mathcal{Y}_t) - f_{0}(\mathcal{Y}_t)\right\} + \lambda \|f\|_1 - \lambda \|f_0\|_1\right| > \frac{\varepsilon}{4}\right).
        \end{split}
    \end{equation}
    
    \item[Step-2:] The first part will converge to zero using the WLLN, but for the second part, using Markov's Inequality, we need to prove
    \[    E_{\epsilon}\left( \sup_{f \in \mathcal{F}_{r,n}} \left|\frac{1}{n} \sum_{t=1}^{n} \epsilon_{t} \left\{ f(\mathcal{Y}_t) - f_{0}(\mathcal{Y}_t)\right\} + \lambda \|f\|_1 - \lambda \|f_0\|_1\right|\right) \to 0\; \text{as}\; n\to \infty.    \]
    
    \item[Step-3:] According to symmetrization for \(\alpha\)-Mixing spatial surfaces with grid spacing \(\eta_{n}  \to 0\) as \(n \to \infty\),
    \begin{equation}
        \begin{split}
            &E_{\epsilon} \left\| \mathbb{P}_{n,\lambda} - P_{\lambda} \right\|_{\mathcal{F}_{r,n}} 
            = E_{\epsilon}\left( \sup_{f \in \mathcal{F}_{r,n}} \left|\frac{1}{n} \sum_{t=1}^{n} \epsilon_{t} \left\{ f(\mathcal{Y}_t) - f_{0}(\mathcal{Y}_t)\right\} + \lambda \|f\|_1 - \lambda \|f_0\|_1\right|\right)\\
            &\leq 2 E_{\epsilon}E_{\rho} \sup_{f \in \mathcal{F}_{r,n}} \left|\frac{1}{n} \sum_{t=1}^{n} \rho_t\epsilon_{t} \left\{ f(\mathcal{Y}_t) - f_{0}(\mathcal{Y}_t)\right\} + \lambda \|f\|_1 - \lambda \|f_0\|_1\right|,
        \end{split}
    \end{equation}
    where \(\rho_t;\;t=1,2,...,n\) are Rademacher random variables independent of \(\epsilon_{1}, \epsilon_{2},..., \epsilon_{n}\). For fixed values of \(\epsilon_{1}, \epsilon_{2},..., \epsilon_{n}\) (conditionally), \(\sum_{t=1}^{n} \rho_t\epsilon_{t} \left\{ f(\mathcal{Y}_t) - f_{0}(\mathcal{Y}_t)\right\}\) is a sub-Gaussian process.
    
    \item[Step-4:] Using the definition of the covering number and Corollary 2.2.8 of \cite{wellner2013weak}, we obtain a universal constant \(K\). Then, using Dudley's integral entropy bound in VC-dimension \(\nu\) with any \(f_{n}^{*} \in \mathcal{F}_{r,n}\) with \(n \geq N_1\), we get
    \begin{equation}
        \begin{split}
            &E_{\rho} \sup_{f \in \mathcal{F}_{r,n}} \left|\frac{1}{n} \sum_{t=1}^{n} \rho_t\epsilon_{t} \left\{ f(\mathcal{Y}_t) - f_{0}(\mathcal{Y}_t)\right\} + \lambda \|f\|_1 - \lambda \|f_0\|_1 \right| \\
            &\leq E_{\rho} \left|\frac{1}{n} \sum_{t=1}^{n} \rho_t\epsilon_{t} \left\{ f_{n}^{*}(\mathcal{Y}_t) - f_{0}(\mathcal{Y}_t)\right\} \right| \\
            &+ K\sqrt{\frac{\nu}{n}}\int_{0}^{2V_1V_2}\sqrt{\frac{\log\;\mathcal{N}(\varphi, \mathcal{F}_{r,n}, \|\cdot\|_{\infty})}{n}} d\varphi.
        \end{split}
    \end{equation}
    
    \item[Step-5:] The first term in the inequality can be handled similarly as in the original proof, using the concept of universal approximation. The second term will now include the effect of the \(\ell_1\) regularization, leading to
    \begin{equation}
        \begin{split}
            &E_{\rho} \sup_{f \in \mathcal{F}_{r,n}} \left|\frac{1}{n} \sum_{t=1}^{n} \rho_t\epsilon_{t} \left\{ f(\mathcal{Y}_t) - f_{0}(\mathcal{Y}_t)\right\} \right| \\
            &+ K\sqrt{\frac{\nu}{n}}\int_{0}^{2V_1V_2}\sqrt{\frac{\log\;\mathcal{N}\left(\frac{\varphi}{2V_1V_2\sqrt{\sigma^{2} + 1}}, \mathcal{F}_{r,n}, \|\cdot\|_{\infty}\right)}{n}} d\varphi \\
            &+ \lambda \sup_{f \in \mathcal{F}_{r,n}} \|f\|_1 - \lambda \|f_0\|_1.
        \end{split}
    \end{equation}
    
    \item[Step-6:] We proceed similarly to the original proof, but now also controlling the effect of the \(\ell_1\) penalty. Assuming \(\lambda = o(n^{-\alpha})\) for some \(\alpha > 0\), the regularization term will diminish as \(n \to \infty\).
    
    \item[Step-7:] Combining all steps, including the regularization term, we get
    \begin{equation}
        \begin{split}
            &E_{\epsilon}\left( \sup_{f \in \mathcal{F}_{r,n}} \left|\frac{1}{n} \sum_{t=1}^{n} \epsilon_{t} \left\{ f(\mathcal{Y}_t) - f_{0}(\mathcal{Y}_t)\right\} + \lambda \|f\|_1 - \lambda \|f_0\|_1\right|\right)\\
            &\leq 2 E_{\epsilon}E_{\rho} \sup_{f \in \mathcal{F}_{r,n}} \left|\frac{1}{n} \sum_{t=1}^{n} \rho_t\epsilon_{t} \left\{ f(\mathcal{Y}_t) - f_{0}(\mathcal{Y}_t)\right\} \right| \\
            &+ 8\sqrt{2}K\sqrt{\frac{\nu C_{r,p,\Gamma_n,V_1, V_2}}{n^{2}}} + \lambda \sup_{f \in \mathcal{F}_{r,n}} \|f\|_1 - \lambda \|f_0\|_1.
        \end{split}
    \end{equation}
    
    \item[Step-8:] To prove convergence, the expression should tend to zero as \(n \to \infty\). The first part converges to zero by the universal approximation theorem. The second part also converges to zero under the assumptions \(\Gamma_n = o(n^{\beta}), r = o(n^{\psi})\), and \(\psi + \beta < 1\). The third term involving \(\lambda\) converges to zero by the choice of \(\lambda\) as \(o(n^{-\alpha})\).
\end{enumerate}





\end{appendices}


\bibliography{sn-bibliography}

\end{document}